\documentclass[lettersize,journal]{IEEEtran}

\makeatletter
\def\endthebibliography{%
	\def\@noitemerr{\@latex@warning{Empty `thebibliography' environment}}%
	\endlist
}
\makeatother

\ifCLASSOPTIONcompsoc
\usepackage[nocompress]{cite}
\else
\usepackage{cite}
\fi
\ifCLASSINFOpdf
\usepackage[pdftex]{graphicx}
\else
\fi
\usepackage{paralist}
\usepackage{booktabs} 
\usepackage{mathtools}
\usepackage{amsfonts}
\usepackage{graphicx,hyperref}
\usepackage{subfigure}
\usepackage{amssymb}
\usepackage{amsthm}
\usepackage{amsmath}
\usepackage{color}
\usepackage[ruled,vlined]{algorithm2e}
\usepackage{algpseudocode}

\newtheorem{theorem}{Theorem}
\newtheorem{lemma}{Lemma}

\newtheorem{definition}{Definition}

\newtheorem{remark}{Remark}

\newcommand{\ad}{\mathrm{ad}}
\newcommand{\Ad}{\mathrm{Ad}}

\newcommand{\rvline}{\hspace*{-\arraycolsep}\vline\hspace*{-\arraycolsep}}

\begin{document}
	%
	\title{Closed-form Error Propagation on the $SE_n(3)$ Group for Invariant Extended Kalman Filtering with  Applications to VINS}

	\author{Xinghan Li, Haodong Jiang, Xingyu Chen, He Kong and Junfeng Wu 
	\thanks{ X. Li is with the College of Control Science and Engineering, and the State Key Laboratory of Industrial Control Technology, Zhejiang University, Hangzhou, P. R. China, Email: xinghanli0207@gmail.com. H. Jiang, X. Chen, and J. Wu are with the School of Data Science,  The Chinese University of Hong Kong, Shenzhen, Shenzhen, P. R. China, Emails:  haodongjiang@link.cuhk.edu.cn, xingyuchenmike@gmail.com,  junfengwu@cuhk.edu.cn. H. Kong is with the Department of Mechanical and Energy Engineering, Southern University of Science and Technology, Shenzhen, P. R. China, Email: kongh@sustech.edu.cn.
	}
}
	
	\maketitle
 	
	\begin{abstract}
	Pose estimation is important for robotic perception, path planning, etc. Robot poses can be modeled on matrix Lie groups and are usually estimated via filter-based methods. In this paper, we establish the closed-form formula for the error propagation for the Invariant extended Kalman filter (IEKF) in the presence of random noises and apply it to vision-aided inertial navigation. We evaluate our algorithm via numerical simulations and experiments on the OPENVINS platform.
	Both simulations and the experiments performed on the public EuRoC MAV datasets demonstrate that our algorithm outperforms some state-of-art filter-based methods such as the quaternion-based EKF, first estimates Jacobian EKF, etc.     
	\end{abstract}
	

	%
	\IEEEpeerreviewmaketitle

\section{Introduction}
	\IEEEPARstart{T}{he} pose estimation problem is highly nonlinear and draws significant research interest. Despite many excellent works based on advanced optimization algorithms~\cite{Leutenegger2015, Campos2021}, the extended Kalman filter (EKF) is recognized as the most widely used method with remarkable merits in terms of computational efficiency and real-time response\cite{Bloesch2017,Hartley2020, Geneva2020ICRA}. The EKF features linearization of the state-space model about the current state estimate, and inaccurate estimates would degrade the convergence~\cite{Krener2003TheCO} and consistency~\cite{Huang2007,Huang2008} of the EKF.
	
	Recently,  the development of invariant observer design harnesses symmetries  (invariance of a  dynamical system to some group actions) of nonlinear systems for improved estimation performance \cite{Bonnabel2008, Bonnabel2007}. In Barrau and Bonnabel's recent works\cite{Barrau,Barrau2018CDC}, they analyzed the convergence of the Invariant extended Kalman filtering (IEKF) for noise-free systems, investigating the log-linear property of the invariant error propagation, and revealing the optimality of IEKF. {They also proposed error propagation formulas in noisy system via approximating the matrix commutator operation. Later in~\cite{Martin2021,Yulin2022}, Martin et al. and Yulin et al. proposed the invariant error propagation in the discrete-time form on $SE_2(3)$ (a matrix Lie group to be discussed in Section~\ref{sec:preliminary}) via the approximation to the \textit{Baker–Campbell–Hausdorff} formula. 
	
		This paper studies the invariant error on the special Euclidean group and propagation of the logarithm of the invariant error in the Lie algebra. To the best of our knowledge, this is the first work that presents the closed-form formula of the invariant error propagation in noisy systems without any approximation.  Our main contributions are as follows: \begin{enumerate}
		\item [(i).] We establish the closed-form formula of the invariant error propagation for noisy systems without any approximation on the Lie algebra of $SE_n(3)$ and customize the results on $SE_2(3)$. These results provide a theoretical basis to analyze the uncertainty propagation. In addition, we elucidate that the only term that makes this stochastic system nonlinear in the Euclidean Space is the noise term. 
		\item [(ii).]  We demonstrate the usefulness of the closed-form formula in analyzing and modeling  uncertainty propagation in the Inertial Measurement Unit (IMU). First, in the presence of small invariant errors or small IMU bias, we reveal that the exact error propagation can be approximated by a linear stochastic system, which is consistent with the result in~\cite{Kanzhi2017} by linearization. Moreover, we propose an approximate but simple method of computing the error covariance matrix for propagation through the introduction of an additional random variable imitating the invariant error to compensate the Jacobian term.  
		\item [(iii).] We  customize the right invariant error result on the $SE_2(3)$ group to IMU and design two  filtering algorithms (with/without Jacobian compensation for the calculation of error covariance matrix) for state estimation of a mobile robot
		with the IMU in a realtime vision-aided inertial navigation system (VINS). We show that two filtering algorithms improve consistency via the observability analysis in~\cite{Huang2008}. The experimental results illustrate that the proposed filters  perform better than state-of-the-art filter-based methods. 
\end{enumerate}}

We organize the paper as follows. Section~\ref{sec:preliminary} revisits the selected preliminary of the Lie group. Section~\ref{sec:main result} established the closed-form expression of the invariant error propagation on the extension of the special Euclidean group. Section~\ref{sec:VINS} applies the results to IMU navigation models and presents filter design in the VINS. Section~\ref{sec:evaluation} reports the simulation and experiment results. Section~\ref{sec:conclusion and future work} concludes and envisions future work. Support lemmas and closely related definitions are given in the Appendix.

	\section{Geometry of Some Matrix Groups} \label{sec:preliminary}
{In this section, we review  Lie groups for robotics~\cite[Chap. 7]{barfoot_2017}.} 	The motion of a rigid body with respect to some reference frame can be described by an element of the special Euclidean group. It is comprised of a rotation
	matrix $\mathbf{R}\in\mathbb R^{3\times 3}$, and a translation vector
	$\mathbf t\in\mathbb R^3$ in the following form:
	\begin{align*}
		\begin{pmatrix}
			\mathbf{R}& \mathbf{t} \\
			\mathbf{0} &	1
		\end{pmatrix}.
	\end{align*}
	In this work we use $SO(3)$ to denote the set of rotation matrices, i.e., the special orthogonal group, and $SE(3)$ to denote the special Euclidean group. 
	The Lie algebra of $SE(3)$ is denoted by $\mathfrak{se}(3)$, and an element of it has the form
	\begin{equation}\label{eqn:se3}
		\begin{pmatrix}
			\omega^\wedge & \mathbf{v} \\
			\mathbf{0} & 				0
		\end{pmatrix},
	\end{equation}
	where $\omega^\wedge$ is the matrix representation of a vector $\omega\in\mathbb R^3$ in the Lie algebra of 
	$SO(3)$ using the ``hat'' operator\footnote{If not specifically in the $SO(3)$ group, the ``hat'' operation $(\cdot)^\wedge$ will denote a linear mapping that the Lie algebra $\mathfrak{g}$ can be identified into $\mathbb{R}^{\text{dim}(\mathfrak{g})}$.  }:
	\begin{align*}
		\omega^\wedge=\begin{pmatrix}
			0&-\omega_3&\omega_2\\
			\omega_3&0&-\omega_1\\
			-\omega_2&\omega_1&0
		\end{pmatrix}\in \mathfrak{so}(3).
	\end{align*}
	The representation of~\eqref{eqn:se3} is also expressed compactly as $(\omega,\mathbf{v})$ for ease of notation.
	
	The special orthogonal group and special Euclidean group are examples of matrix Lie groups. Consider a matrix Lie group $G$ and its Lie algebra $\mathfrak{g}$.
	The Lie bracket of  $\mathbf x$ and $\mathbf y$ in $\mathfrak{g}$ is given by the matrix commutator~$\left[ \mathbf{x},\mathbf{y}\right]=\mathbf{x}\mathbf{y}-\mathbf{y}\mathbf{x} $. In particular, the Lie bracket on $\mathfrak{so}(3)$ corresponds to the outer products of the two vectors generating the matrix elements of $\mathfrak{so}(3)$ via the ``hat'' operator, i.e., $\left[ \omega_1, \omega_2\right]=\omega_1\times\omega_2=(\omega_1^\wedge)\omega_2 $, and the Lie bracket on $\mathfrak{se}(3)$ is given by 
	\begin{align*}
		\left[ (\omega_1,\mathbf{v}_1),(\omega_2,\mathbf{v}_2)\right]=(\omega_1\times\omega_2,\omega_1\times\mathbf{v}_2-\omega_2\times\mathbf{v}_1) .
	\end{align*} 
	The adjoint action of $\mathbf{X}\in G$ on the Lie algebra
	$\mathbf{y}\in\mathfrak g$ is defined as $\Ad_\mathbf{X}:\mathfrak{g}\rightarrow\mathfrak{g},  \mathbf{y}\mapsto\mathbf{X}\mathbf{y}\mathbf{X}^{-1}$. In particular, for $\mathbf{X}:=(\mathbf{R},\mathbf{t})\in SE(3)$, the adjoint map on $\mathbf{y}:=(\omega,\mathbf{v})\in\mathfrak{se}(3)$ is given by 
	\begin{equation*}
		\Ad_\mathbf{X}(\mathbf{y})=\left(\mathbf{R}\omega,\mathbf{t}\times\mathbf{R}\omega+\mathbf{R}\mathbf{v} \right),
	\end{equation*}
	admitting a matrix representation
	\begin{equation*}
		\Ad_\mathbf{X}(\mathbf{y})=
		\begin{pmatrix}
			\mathbf{R}&\mathbf{0}\\
			\mathbf{t}^\wedge\mathbf{R}&\mathbf{R}
		\end{pmatrix}\begin{pmatrix}
			\mathbf{\omega}\\
			\mathbf{v}
		\end{pmatrix}.
	\end{equation*}
	The differential of the adjoint action $\Ad_{\mathbf X}$ at the identity element of $G$, denoted as $\ad_{\mathbf{x}}:\mathfrak{g}\rightarrow\mathfrak{g}$, is a linear mapping from $\mathfrak g$ to itself,
	which defines the adjoint action of  $ {\mathfrak {g}}$ on itself. Moreover, for $\mathbf{x},\mathbf{y}\in\mathfrak{g}$, it holds that $\ad_{\mathbf{x}}(\mathbf y)=\left[\mathbf{x}, \mathbf{y} \right]$. In particular, for $\mathbf{x}:=(\omega_1,\mathbf{v}_1)$ and $\mathbf{y}:=(\omega_2,\mathbf{v}_2)$ in $\mathfrak{se}(3)$,  
	\begin{align*}
		\ad_{\mathbf{x}}\left( \mathbf{y}\right) =(\omega_1\times\omega_2,\omega_1\times\mathbf{v}_2-\omega_2\times\mathbf{v}_1),
	\end{align*}
	which also admits a matrix representation
	\begin{equation*}
		\ad_{\mathbf{x}}\left( \mathbf{y}\right)=\begin{pmatrix}
			\omega_1^\wedge&\mathbf{0}\\
			\mathbf{v_1}^\wedge&\omega_1^\wedge
		\end{pmatrix}\begin{pmatrix}
			\mathbf{\omega_2}\\
			\mathbf{v_2}
		\end{pmatrix}.
	\end{equation*}
	The group and the Lie algebra are linked to one another via the exponential and logarithm operations. To be specific, the exponential map of a matrix, $\exp: \mathfrak{g}\to G$, allows us to wrap
	a Lie algebra element around the group. 
	The exponential map is locally invertible at some neighborhood of the identity of the Lie algebra, in which it is injective. We then can define the logarithm mapping $\log: G\to \mathfrak{g}$ as the inverse of the exponential mapping within such a neighborhood. It in turn acts as an unwrapping operation from the group to the Lie algebra. 
	
	At last, we introduce the matrix Lie group  $SE_n(3)$ as an extension of $SE(3)$, which is comprised of a rotation
	matrix in $\mathbb R^{3\times 3}$ and $n$ vectors in $\mathbb{R}^3$.
	The matrix representations of an element in $SE_n(3)$ and an element in its Lie algebra  $\mathfrak{se}_n(3)$ are as follows:
	\begin{align*}
		\begin{pmatrix}
			\begin{matrix}
				\mathbf{R}
			\end{matrix}
			& \rvline & \begin{matrix}
				\mathbf{t}_1&\cdots&\mathbf{t}_n
			\end{matrix} \\
			\hline
			\mathbf{0} & \rvline &
			\begin{matrix}
				\mathbf{I}
			\end{matrix}
		\end{pmatrix}, \quad 
		\begin{pmatrix}
			\omega^\wedge
			& \rvline & \begin{matrix}
				\mathbf{v}_1&\cdots&\mathbf{v}_n
			\end{matrix}\\
			\hline
			\mathbf{0} & \rvline &
			\begin{matrix}
				\mathbf{0}
			\end{matrix}
		\end{pmatrix}.
	\end{align*}
	The matrix representations of the adjoint action of $SE_n(3)$ on  $\mathfrak{se}_n(3)$  and the the adjoint action of $\mathfrak{se}_n(3)$ on itself are as follows:
	\begin{align}\label{eqn:sen3_matrix_form}
		\begin{pmatrix}
			\mathbf{R}&\mathbf{0}&\mathbf{0}&\mathbf{0}\\
			\mathbf{t}_1^\wedge\mathbf{R}&\mathbf{R}&\mathbf{0}&\mathbf{0}\\
			\vdots&\mathbf{0}&\ddots&\mathbf{0}\\
			\mathbf{t}_n^\wedge\mathbf{R}&\mathbf{0}&\mathbf{0}&\mathbf{R}
		\end{pmatrix}, \quad 
		\begin{pmatrix}
			\omega^\wedge&\mathbf{0}&\mathbf{0}&\mathbf{0}\\
			\mathbf{v}_1^\wedge&\omega^\wedge&\mathbf{0}&\mathbf{0}\\
			\vdots&\mathbf{0}&\ddots&\mathbf{0}\\
			\mathbf{v}_n^\wedge&\mathbf{0}&\mathbf{0}&\omega^\wedge
		\end{pmatrix}.
	\end{align}	
{
	Specifically, the group $SE_2(3)$ orignated from~\cite{Barrau}, shares the common structure of $SE_n(3)$, which can represent an orientation $\mathbf{R}$, position $\mathbf{p}$ and velocity $\mathbf{v}$ of a rigid body:
	\begin{equation*}
		SE_2(3):=\left\lbrace \mathbf{X}=\begin{pmatrix}
\mathbf{R}&\vline &\mathbf{v}&\mathbf{p}\\
\hline
\mathbf{0}_{2\times3}&\vline&\mathbf{I}_{2}
		\end{pmatrix}\vline \mathbf{R}\in SO(3),\mathbf{p},\mathbf{v}\in\mathbb{R}^{3} \right\rbrace .
	\end{equation*}}
	
\section{ Logarithmic Invariant Error Propagation }\label{sec:main result}
In this section, we present the closed-form error dynamics on $\mathfrak{se}_n(3)$ without any approximation and elucidate that the nonlinearity of error propagation dynamics is caused by the introduction of the noise to the system.
\subsection{Uncertainty Propagation in Dynamics}

Consider the following system on a matrix group $SE_n(3)$:
\begin{equation}\label{eqn:dynamic_sen(3)}
	\dot{\mathbf{X}}= \mathbf{X}\mathbf{v}_b+\mathbf{v}_g\mathbf{X}+f_0(\mathbf{X}), \hbox{~with~}\mathbf X(0)=\mathbf{X}_0,
\end{equation}
where $\mathbf{X} \in SE_n(3)$ is the system state, $\mathbf{v}_b,\mathbf{v}_g\in  \mathfrak{se}_n(3)$ are inputs, and $f_0$ is a vector filed, a smooth mapping from $SE_n(3)$ to its tangent bundle $T(SE_n(3))$.
In practice, $\mathbf{X}$ represents the configuration of a robot action space, $\mathbf v_b$ represents the input in the body frame, and $\mathbf v_g$ represents the input in the fixed frame.

We consider another dynamical system governed by the same vector field $f_0$, but differently initialized and steered by a noisy input $\mathbf v_{b}^{(n)}$, that is,
\begin{equation}\label{eqn:dynamics_bar_X}
	\dot{\bar{\mathbf{X}}}= \bar{\mathbf{X}}\mathbf{v}_{b}^{(n)}+\mathbf{v}_g\bar{\mathbf{X}}+f_0(\bar{\mathbf{X}}), \hbox{~with~} \bar{\mathbf{X}}(0)=\bar{\mathbf X}_0,
\end{equation}
$\mathbf{v}_{b}^{(n)}=\mathbf{v}_b+ \mathbf w$ where $\mathbf w^\vee\in \mathbb R^{3(n+1)}$ is assumed to be a white noise\footnote{The operation $(\cdot)^\vee$, called ``vee'' operation,  is the inverse of the ``hat'' operation, which is a mapping from $\mathfrak{se}_n(3)$ to $\mathbb{R}^{3(n+1)}$.}.
The trajectory  $\bar{\mathbf{X}}$ performs an open-loop tracking of ${\mathbf{X}}$. The input $\mathbf v_b^{(n)}$ is from body-fixed motion sensors, such as accelerates and perimeters, and $\mathbf w$ then captures the sensor noise. 

Following~\cite{Barrau}, we introduce two forms of errors between trajectories $\mathbf X$ and $\bar{\mathbf X}$ in the sense of group multiplication, which can be viewed as the group analog of linear errors in a vector space. 
\begin{definition}[Invariant Errors]
	The left (right) invariant error of the trajectories $\mathbf X$ and 
	$\hat {\mathbf X}$ is defined respectively as follows:
	\begin{align}\label{eqn:invariant_error}
		{\eta}_{L} &:= {\mathbf{X}}(t)^{-1}\bar{\mathbf{{X}}}(t) \quad \text{(left invariant error)},\\
		{\eta}_{R} &:= \bar{\mathbf{{X}}}(t)\mathbf{{X}}(t)^{-1} \quad \text{(right invariant error)}.
	\end{align}
\end{definition}

The paper~\cite{Barrau} studies a class of 
group affine dynamics generalizing linear systems, which are of the form
$$\dot{
	\mathbf X}=f_{u}(\mathbf X),
$$
where $f_{u}: SE_n(3)\to T(SE_n(3))$ satisfies the so-called group affine property
for any $u\in\mathcal U,~\mathbf{X}_1, \mathbf{X}_2\in SE_n(3)$,~ 
\begin{equation}\label{eqn:affine_linear_property}
	f_u(\mathbf{X}_1\mathbf{X}_2)=
	\mathbf{X}_1 f_u(\mathbf{X}_2)+
	f_u(\mathbf{X}_1)\mathbf{X}_2-\mathbf{X}_1f_u(I)
	\mathbf{X}_2.
\end{equation}
Different from~\eqref{eqn:affine_linear_property}, in this paper we require that $f_0(\cdot)$ satisfies  
\begin{equation}\label{eqn:f0_assumption}
	f_0(\mathbf{X}_1\mathbf{X}_2)=f_0(\mathbf{X}_1)\mathbf{X}_2+\mathbf{X}_1f_0(\mathbf{X}_2).
\end{equation}
One can verify that with $f_0(\cdot)$ satisfying~\eqref{eqn:f0_assumption},  the dynamics of the invariant errors becomes:
\begin{align}	\dot{{\eta}}_{L}&=-\mathbf{v}_b{\eta_{L}}+{\eta_{L}}\mathbf{v}_b+{\eta_{L}}\mathbf{w}+f_0({\eta_{L}}), \text{(left invariant)}\label{eqn:dynamic_left_invariant_error}\\
	\dot{{\eta}}_{R}&=\mathbf{v}_g{\eta_{R}}-{\eta_{R}}\mathbf{v}_g+\Ad_{\bar{\mathbf{X}}}\mathbf{{w}}{\eta_{R}}+f_0({\eta_{R}}), \text{(right invariant)}\label{eqn:dynamic_right_invariant_error}.
\end{align}
Note that the dynamics of~\eqref{eqn:dynamic_sen(3)} satisfies the group affine property~\eqref{eqn:affine_linear_property} and the left-invariant dynamics~\eqref{eqn:dynamic_left_invariant_error} and the right-invariant dynamics~\eqref{eqn:dynamic_right_invariant_error} satisfy~\eqref{eqn:f0_assumption}.

Since the Lie algebra can be thought of as infinitesimal motions near the identity of a Lie group, we define an error vector $\xi\in\mathbb{R}^{3(n+1)}$ called the  logarithmic invariant error as follows:
\begin{equation}\label{eqn:lie_log_error}
	\xi=\log(\eta)^\vee.
\end{equation}
For ease of notations, for an error $\eta\in SE_n(3)$, let $\xi:=\begin{pmatrix}
	\omega&\mathbf{v}_1&\cdots&\mathbf{v}_n
\end{pmatrix}^\top$ denote the corresponding Lie logarithm of it, where
$\omega^\wedge\in \mathfrak{so}_3$ and $\mathbf{v}_i\in\mathbb{R}^3$.  In what follows, we give our main results on the dynamical evolution of $\xi$ for both the left and the right invariant errors.
{
	\begin{theorem}\label{thm: log_linear}
		Consider the  dynamics~\eqref{eqn:dynamic_left_invariant_error}\eqref{eqn:dynamic_right_invariant_error} of ${\eta}$ where
		$f_0$ satisfies~\eqref{eqn:f0_assumption}. The dynamics of 
		$\xi$ in~\eqref{eqn:lie_log_error} is given by:\\
		Left invariant	
		\begin{equation}\label{eqn:left-invariant-error}
			\dot{\xi}_{L}=-\ad_{\mathbf{v}_b}\xi_{L}+\mathbf{J}(-\ad_{\xi_{L}^\wedge})^{-1}\mathbf{w}^\vee+\mathbf{A}\xi_{L},
		\end{equation}
		Right invariant
		\begin{equation}\label{eqn:right-invariant-error}
			\dot{\xi}_{R}=\ad_{\mathbf{v}_g}\xi_{R}+\mathbf{J}(\ad_{\xi_{R}^\wedge})^{-1}\Ad_{\hat{\mathbf{X}}}\mathbf{{w}}^\vee+\mathbf{A}\xi_{R},
		\end{equation}
		where $\mathbf{A}=\frac{\partial}{\partial\xi}f_0(\exp{(\xi)})$ and $\mathbf{J}(\ad_\mathbf{x}):=\sum_{i=0}^{\infty}\frac{1}{(i+1)!}(\ad_{\mathbf{x}})^{i}$ which is defined in Definition~\ref{def:Jacobians} in the Appendix, called the left Jacobians of $\mathbf{x}$ for $\mathbf{x}\in\mathfrak{se}_n(3)$.    
	\end{theorem}
	\begin{proof}
	We only give a proof for the left invariant error case.  For the right invariant case, it is similar and skipped. 
	
	Denote $\xi:=\begin{pmatrix}
		\omega&\mathbf{v}_1&\ldots&\mathbf{v}_n
	\end{pmatrix}$.
	By~\eqref{eqn:sen3_matrix_form}, the matrix  $\ad_{\xi^\wedge}$ has eigenvalues up to multiplicity as follows:
	\begin{equation*}
		\lambda_1 = 0 \quad \lambda_2 = \left| \omega\right|\mathbf i \quad \lambda_3=-\left| \omega\right|\mathbf i.	
	\end{equation*}
	From Lemma~\ref{lemma:derivative_of_lie_group} together with~\eqref{eqn:dynamic_left_invariant_error} , we obtain that  
	\begin{equation*}
		\begin{split}
			&{\eta}(t)^{-1}\dot{{\eta}}(t)=\sum_{i=0}^{\infty}\frac{(-1)^i}{(i+1)!}(\ad_{\xi^\wedge})^i\dot{\xi}^\wedge\\
			&\text{i.e.,}\quad \mathbf{v}_b-\exp{\left( -\xi^\wedge\right) }\mathbf{v}_b\exp{\left( \xi^\wedge\right) }\\ &+\mathbf{w}+\exp{(-\xi^\wedge)}f_0(\exp{(\xi^\wedge)})=\sum_{i=0}^{\infty}\frac{(-1)^i}{(i+1)!}(\ad_{\xi^\wedge})^i\dot{\xi}.
		\end{split}
	\end{equation*}
	By Lemma~\ref{lemma:Adjoit_representation_of_B} and adjoint represent of a Lie algera, the above equation writes
	\begin{equation*}
		\left( \mathbf{I}-\text{e}^{-\ad_{\xi^\wedge}}\right)\mathbf{v}_b^\vee+\mathbf{w}^\vee+\left( \exp{(-\xi)}f_0(\exp{(\xi)})\right)^\vee= \frac {\mathbf{I}-e^{-\mathrm {ad} _{\mathbf{\xi^\wedge}}}}{\mathrm {ad} _{\mathbf{\xi^\wedge}}}\dot{\xi}.
	\end{equation*}
	Since the logarithm is defined within the injectivity radius around the identity of $\mathfrak{se}_n(3)$, we have $|\omega|\neq2k\pi, k=1,2,\ldots$ which together with Lemma~\ref{lemma:invertible matrix} yields that 
	\begin{equation}\label{eqn:aeqnoftheorem1}
		\dot{\xi}=\ad_{\xi^\wedge}\mathbf{v}_b^\vee+\left( \frac {\mathbf{I}-e^{-\mathrm {ad} _{\xi^\wedge}}}{\mathrm {ad} _{\xi^\wedge}}\right)^{-1}\mathbf{w}^\vee+g_0(\xi)
	\end{equation} 
	where $g_0(\xi):=\left(\eta(t) \frac {\mathbf{I}-e^{-\mathrm {ad} _{\xi(t)}}}{\mathrm {ad} _{\xi(t)}}\right)^{-1}f_0(\exp{(\eta)})$.  Now we will show that the mapping $g_0(\cdot):x\mapsto \mathbf{A}x$ is a linear mapping. Consider a state  $\eta^\star\in SE_n(3)$ which satisfies $\dot{\eta^\star}=f_0(\eta^\star)$ and a vector of its logarithm denoted as $\xi^\star$. By~Lemma~\ref{lemma:linear_invariant_error} and the property of $f_0$~\eqref{eqn:f0_assumption}, the  vector $\xi^\star$ have the following dynamic
	\begin{equation*}
		\dot{\xi}^\star=\mathbf{A}\xi^\star
	\end{equation*}
	where $\mathbf{A}$ can be computed by $f_{0}(exp(\xi^\star))=\mathbf{A}\xi^\star+O(\left\|\xi^\star \right\|^2 )$. From~Lemma~\ref{lemma:derivative_of_lie_group}, the vector $\xi^\star$ writes
	\begin{equation*}
		\dot{\xi}^\star = g_0(\xi^\star).
	\end{equation*}
	Therefore,  we conclude that $g_0(\xi^\star)=\mathbf{A}\xi^\star$. The equation~\eqref{eqn:aeqnoftheorem1} writes
	\begin{align*}				\dot{\xi}(t)&=\ad_{\xi^\wedge}\mathbf{v}_b^\vee+\left( \frac {\mathbf{I}-e^{-\mathrm {ad} _{\xi^\wedge}}}{\mathrm {ad} _{\xi^\wedge}}\right)^{-1}\mathbf{w}^\vee+g_0(\xi)\\
		&=-\ad_{\mathbf{v}_b}\xi+\mathbf{J}(-\ad_{\xi^\wedge})^{-1}\mathbf{w}^\vee+\mathbf{A}\xi,
	\end{align*}
	which completes the proof.
\end{proof}}

{
	The closed-form evolution of noisy 
	$\xi(t)$ is provided by Theorem~\ref{thm: log_linear}, which is described by a stochastic differential equation (SDE) in $\mathbb R^{3(n+1)}$. Though it is only the noisy terms in ~\eqref{eqn:left-invariant-error} and~\eqref{eqn:right-invariant-error} that cause the noisy dynamics to be nonlinear, Gaussianity of the distribution of $\xi(t)$ is not closed under the propagation
	given by~\eqref{eqn:left-invariant-error} and~\eqref{eqn:right-invariant-error}. 
	Finding a solution to the SDE is involved, and will be investigated in our future work.

	
	\begin{remark}[Relations with Existing Works]	
		In the existing literature~\cite{Barrau}, the dynamics of the logarithmic invariant error has been shown to satisfy a log-linear property when the dynamics~\eqref{eqn:dynamic_left_invariant_error} and~\eqref{eqn:dynamic_right_invariant_error} are noise free, while in the presence of noise
		$\mathbf w$, the noise term is tackled  by linear approximation~\cite{Barrau}. 	In this paper, we analyze the influence of the noise on the logarithmic invariant error propagation via the derivative of the matrix exponential. When the left Jacobians of the error are approximated to the identity by first-order linearization,
		our result boils down to results in~\cite{Barrau,barrau2018invariant,Hartley2020,Kanzhi2017}. More about the linearization approximation in $SE_2(3)$ will be discussed in Section~\ref{section:Approximation_methods}. 
	\end{remark}	
}	
	
	\section{Applications in VINS}\label{sec:VINS}
We will customize the right invariant error result on the $SE_2(3)$ group to IMU and apply it to analyze the IMU kinematics. Based on these theoretical results, we further design a filtering algorithm for state estimation of a mobile robot with the IMU in the realtime vision-aided inertial navigation scenario.
\subsection{Logarithmic  Right Invariant Error Propagation on the $SE_2(3)$}
We are interested in estimating a rigid body's 3D orientation, position, and velocity in the world frame, given angular velocity and acceleration measurements from the IMU attached to the rigid body. The above collection of variables forms the state variable of IMU and can be represented in the world frame as an element of  $SE_2(3)$. To be specific, the IMU state $\mathbf X_I(t)$ at time $t$ can be expressed as:
\begin{equation*}
	\mathbf{X}_{I}(t)=\begin{pmatrix}\begin{matrix}
			\mathbf{R}(t)
		\end{matrix}
		& \rvline & \mathbf{p}(t)\quad\mathbf{v}(t) \\
		\hline
		\mathbf{0}_{2\times3} & \rvline &
		\begin{matrix}
			\mathbf{I}_{2}
		\end{matrix}
	\end{pmatrix}
\end{equation*} 
where
\begin{enumerate}
	\item [(i)]  $\mathbf{R}(t)\in SO(3)$ represents the rotation of the rigid body relative to the reference frame, reflecting the rotation from the world frame to the body frame attached to the rigid body;
	\item [(ii)]$\mathbf{p}(t)\in\mathbb R^3$ and $\mathbf{v}(t)\in\mathbb R^3$ are the position and velocity relative to the world frame, respectively.
\end{enumerate}
The continuous-time kinematics of the orientation, velocity and position of the IMU are respectively described as follows:
\begin{equation}\label{eqn:model_equation}
	\dot{\mathbf{R}}(t)= \mathbf{R}(t)\omega(t)^\wedge,\quad\dot{\mathbf{p}}(t) = \mathbf{v}(t), \quad \dot{\mathbf{v}}(t) = \mathbf{a}(t),
\end{equation}  
where $\mathbf{\omega}$ denotes the angular velocity relative to the body frame and  $\mathbf{a}$ denotes the acceleration relative to the world frame. Because the measurements of $\omega$ and $\mathbf{a}$ suffer from gyroscope and accelerometer bias, IMU measurements $\omega_m$ and $\mathbf{a}_m$  are usually modeled as the true angular velocity and linear acceleration variables corrupted by additive Gaussian white noise plus measurement biases:  
\begin{align}\label{eqn:IMU measurement}
	\mathbf{\omega}_m(t)  &= \mathbf{\omega}(t)  +  \mathbf{b}_{\omega}(t)  +  \mathbf{n}_\omega(t), \notag\\
	\mathbf{a}_m(t)  &= \mathbf{R}(t)^{-1} \left( \mathbf{a}(t) -\mathbf{g}\right) +  \mathbf{b}_a(t)  + \mathbf{n}_a(t),
\end{align}  
where $\mathbf{g}$ denotes the gravitational acceleration relative to the world frame, and $\mathbf{n}_\omega$ and $\mathbf{n}_a$ are Gaussian noises, and $\mathbf{b}_\omega$ and $\mathbf{b}_a$ are the gyro and accelerometer biases. Typically, the biases are further modeled as stochastic processes driven by white Gaussian noises, that is,
$$\dot{\mathbf{b}}_{\omega}(t) = \mathbf{n}_{b_{\omega}}(t), \quad \dot{\mathbf{b}}_a(t) =\mathbf{n}_{b_{a}}(t).$$
Therefore, using the IMU measurements, the dynamics of the whole state $\mathbf X_{I}(t)$ can be given in a compact form:
\begin{equation}\label{eqn:vins_model}
	\dot{\mathbf{X}}_{{I}}(t)=\mathbf M \mathbf{X}_{{I}}(t) \mathbf N+\mathbf{X}_{{I}}(t)\mathbf{v}_b(t)+\mathbf{v}_g(t)\mathbf{X}_{{I}}(t),
\end{equation}
where $\mathbf{M}:=\begin{pmatrix}
	\mathbf{I}_{3}
	&\mathbf{0}_{3\times2}\\
	\mathbf{0}_{2\times3}&
	\begin{matrix}
		\mathbf{0}_{2}
	\end{matrix}
\end{pmatrix}$ 
and $\mathbf {N}:=\begin{pmatrix}
	\mathbf{0}_{4\times3}
	&\rvline&\mathbf{0}_{4\times2}\\
	\hline
	\mathbf{0}_{1\times3}&\rvline&
	1\quad 0
\end{pmatrix},$ and 	
$\mathbf{v}_b:=\begin{pmatrix}
	(\mathbf{\omega}_m -  \mathbf{b}_{\omega} - \mathbf{n}_\omega ) ^\wedge
	&\rvline&\mathbf{0}_{3\times1}\quad\mathbf{a}_m -  \mathbf{b}_a - \mathbf{n}_a\\
	\hline
	\mathbf{0}_{2\times3}&\rvline&
	\mathbf{0}_{2\times2}
\end{pmatrix}\in\mathfrak{se}_2(3)$ and 
$\mathbf{v}_g:=\begin{pmatrix}
	\mathbf{0}_{3\times3}
	&\rvline&\mathbf{0}_{3\times1}\quad\mathbf{g}\\
	\hline
	\mathbf{0}_{2\times3}&\rvline&
	\begin{matrix}
		\mathbf{0}_{2\times2}
	\end{matrix}
\end{pmatrix}\in\mathfrak{se}_2(3)$
are inputs relative to the body frame and the world frame respectively.
In~\eqref{eqn:vins_model}, $\mathbf M \mathbf{X}_{{I}}(t) \mathbf N$ reflects the ``autonomous'' part of the dynamics.

An estimate $\hat{\mathbf X}_I$ to $\mathbf X_I$ can be propagated by making an estimate to $\mathbf v_b$ from the IMU measurements:
\begin{equation}\label{eqn:estimation}
	\dot{\hat{\mathbf{X}}}_{{I}}(t)=\mathbf M \hat{\mathbf{X}}_{{I}}(t) \mathbf N+ \hat{\mathbf{X}}_{{I}}(t)\hat{\mathbf{v}}_b+\mathbf{v}_g(t) \hat{\mathbf{X}}_{{I}}(t)
\end{equation}
with 
$\hat{\mathbf{v}}_b:=\begin{pmatrix}
	(\mathbf{\omega}_m -  \hat{\mathbf{b}}_\omega  ) ^\wedge
	&\rvline&\mathbf{0}_{3\times1}\quad\mathbf{a}_m-  \hat{\mathbf{b}}_a\\
	\hline
	\mathbf{0}_{2\times3}&\rvline&
	\mathbf{0}_{2}
\end{pmatrix}$.
Note that the estimated biases maintain unchanged during the estimate propagation since the IMU measurement does not support us to refine our estimate to the device's biases, i.e.,  $\dot{\hat{\mathbf{b}}}_{\omega}=\mathbf{0}$ and $\dot{\hat{\mathbf{b}}}_a=\mathbf{0}$ for the propagation. We define the estimation error of the gyro and accelerometer biases as 
$\tilde{\mathbf{b}}_{\omega}:=
\hat{\mathbf{b}}_{\omega}-\mathbf{b}_{\omega}$ and $\tilde{\mathbf{b}}_a:=\hat{\mathbf{b}}_a-
\mathbf{b}_a$. The right invariant errors in~\eqref{eqn:invariant_error} denoted by $\eta_I:=\hat{\mathbf{X}}_{I}\mathbf{X}_{I}^{-1}$ for the state of IMU have the explicitly form:
\begin{equation*}
	\eta_I = \begin{pmatrix}\begin{matrix}
			\tilde{\mathbf{R}}(t)
		\end{matrix}
		& \rvline & \hat{\mathbf{p}}(t)-\tilde{\mathbf{R}}(t)\mathbf{p}(t)\quad\hat{\mathbf{v}}(t)-\tilde{\mathbf{R}}(t)\mathbf{v}(t) \\
		\hline
		\mathbf{0}_{2\times3} & \rvline &
		\mathbf{I}_2
	\end{pmatrix}.
\end{equation*}
We then have the following results for the dynamics of the logarithmic invariant error, $\xi_{I}:=\log( {\eta_{I}})^\vee$.
\begin{theorem}\label{pro:log_err_vins_model}
	Consider the dynamics~\eqref{eqn:vins_model} and~\eqref{eqn:estimation}. The dynamics of $\xi_I\in\mathbb R^{9}$ is given by :
	\begin{equation}\label{eqn:vins_invariant_error_log_dym}
		\dot{\xi_{I}} = \mathbf{A}\xi_{I}+\mathbf{J}(\ad_{{\xi_{I}}^\wedge})^{-1}\mathbf{B} 
		\begin{pmatrix}
			-\mathbf{\tilde{b}}_{\omega}+\mathbf{n}_\omega\\ 
			-\mathbf{\tilde{b}}_a+\mathbf{n}_a
		\end{pmatrix},
	\end{equation}
	where $\mathbf{A}:=\begin{pmatrix}
		\mathbf{0}&\mathbf{0}&\mathbf{0}\\
		\mathbf{0}&\mathbf{0}&\mathbf{I}_{3}\\
		\mathbf{g}^\wedge&\mathbf{0}&\mathbf{0}
	\end{pmatrix}$ and $\mathbf{B}:=\begin{pmatrix}
		\hat{\mathbf{R}}&\mathbf{0}\\
		\hat{\mathbf{p}}^\wedge\hat{\mathbf{R}}&\mathbf{0}\\
		\hat{\mathbf{v}}^\wedge\hat{\mathbf{R}}&\hat{\mathbf{R}}\end{pmatrix}$.
\end{theorem}
\begin{proof}
	The dynamics of right invariant error $\eta_I$ can be computed as   
	\begin{align*}
		\dot{\eta_I}&=\dot{\hat{\mathbf{X}}}_{I}\mathbf{X}_{I}^{-1}-\mathbf{X}_{I}^{-1}\hat{\mathbf{X}}_{I}\dot{\mathbf{X}}_{I}\mathbf{X}_{I}^{-1}\\
		&=\mathbf{v}_g\eta_I-\eta_I\mathbf{v}_g+\mathbf{M}\eta_I\mathbf{N}+\Ad_{\hat{\mathbf{X}}_{I}}\mathbf{W}\eta_I
	\end{align*}
	where $\mathbf{W}:=\begin{pmatrix}
		-\mathbf{\tilde{b}}_{\omega}^{\wedge}+\mathbf{n}_\omega^\wedge &\mathbf{0}&-\mathbf{\tilde{b}}_a+\mathbf{n}_a\\
		\mathbf{0}&\mathbf{0}&\mathbf{0}
	\end{pmatrix}$.	By Theorem~\ref{thm: log_linear}, we get~\eqref{eqn:vins_invariant_error_log_dym}, which completes the proof. 
\end{proof}

{
	\subsection{Approximation Methods for Error Propagation }\label{section:Approximation_methods}
	The error propagation in the Lie algebra is essentially a diffusion process driven by a white random process $\mathbf{w}(t)$ described by~\eqref{eqn:vins_invariant_error_log_dym}. The solution of the SDE depends on $\mathbf{J}(\ad_{{\xi_{I}}^\wedge})^{-1}\Ad_{\hat{\mathbf{X}}_{I}}$, causing mathematical difficulty in analysis. For estimation purposes, we will discuss on the following different approximations for the uncertainty propagation in an IMU. 
	\begin{enumerate}[(i).]
		\item\textit{Under Small Invariant Error:}
		It is usually assumed that the error $\eta_I$ is small enough, that is, $\eta_I$ is sufficiently close to the identity as $\mathbf X_{I}$ tracks $\hat {\mathbf X}_{I}$ well so that the following approximation is viable:
		$\xi_I \approx 0$ and $\mathbf{J}(\ad_{{\xi_{I}}^\wedge})^{-1} \approx \mathbf I$. As a consequence, \eqref{eqn:vins_invariant_error_log_dym} can be significantly simplified. Now we define an augmented state for the IMU device as:
		$
		\begin{pmatrix}
			\mathbf X_I\\
			\mathbf b
		\end{pmatrix},\hbox{~with~}
		\mathbf b:=\begin{pmatrix}
			\mathbf b_{\omega}\\
			\mathbf b_{a}
		\end{pmatrix}.
		$
		Then the augmented error can be approximated in a compact form via adding the simple vector differences $\tilde{\mathbf{b}}$\footnote{
			The augmentation technique is also used for developing 
			``imperfect'' IEKF in~\cite{barrau2015non}. The introduction of additional vector difference
			coins the term of ``imperfect'' IEKF, since it
			sacrifices all the properties of the IEKF, see~\cite{barrau2015non} for details.}:
		\begin{equation}\label{eqn: IMU_augment_state_propagation}
			\begin{pmatrix}
				\dot{\xi_{I}}\\
				\dot{\tilde{\mathbf{b}}}
			\end{pmatrix}\approx\begin{pmatrix}
				\mathbf{A}
				& -\mathbf{B} \\
				\mathbf{0}  &
				\mathbf{0}
			\end{pmatrix}
			\begin{pmatrix}
				\xi_{I}\\
				\tilde{\mathbf{b}}
			\end{pmatrix}+\begin{pmatrix}
				\mathbf{B}
				& \mathbf{0} \\
				\mathbf{0}  &
				\mathbf{I}_2
			\end{pmatrix}
			\mathbf{n}_{\rm imu}
		\end{equation}			
		where $\mathbf{n}_{\rm imu}:=\left( \mathbf{n}_\omega^\top~ \mathbf{n}_a^\top~\mathbf{n}_{b_{\omega}}^\top~ \mathbf{n}_{b_a}^\top\right)^\top $ denotes the total IMU internal device noise, which are assumed to follow a normal distribution~$\mathcal{N}(0,\mathbf{Q}_{\rm imu})$. Note that the covariance $\mathbf{Q}_{\rm imu}$ of $\mathbf{n}_{\rm imu}$ is usually accessible through sensor calibration. When the bias error handling the bias and measurement noise denoted as $\mathbf{\tilde{b}}_{i}-\mathbf{n}_{i}$ in~\eqref{eqn:vins_invariant_error_log_dym} for $i\in\left\lbrace \omega, a \right\rbrace $ is sufficiently small, the closed-form formula can also be approximated as~\eqref{eqn: IMU_augment_state_propagation}. The propagation result is essentially the same with the uncertainty propagation proposed in~\cite{Hartley2020} and~\cite{Kanzhi2017}. We derive the result from a different formula, and it clarifies how the bias error affects the uncertainty propagation.
		\item \textit{Imitating Jacobians of Invariant Error to Compensate Error Covariance Propagation:} Approximation as we do in (i)  for~\eqref{eqn: IMU_augment_state_propagation} leads to inaccurate covariance propagation when $\xi_I$ is large. To tackle this problem, we propose to use an additional variable, termed $\xi_\delta$, drawn independently from some certain distribution to replace $\xi_I$ in $\mathbf{J}(\ad_{\xi^\wedge})$. By doing so, 
		we have a new approximation of~\eqref{eqn:vins_invariant_error_log_dym} as follows:
		\begin{equation}\label{eqn:correct_error_propagation}
			\begin{pmatrix}
				\dot{\xi_{I}}\\
				\dot{\tilde{\mathbf{b}}}
			\end{pmatrix}\approx
			\underbrace{ \begin{pmatrix}
					\mathbf{A}
					& -\mathbf{J}_\delta^{-1}\mathbf{B} \\
					\mathbf{0}  &
					\mathbf{0}
			\end{pmatrix}}_\mathbf{F}
			\begin{pmatrix}
				\xi_{I}\\
				\tilde{\mathbf{b}}
			\end{pmatrix}+\underbrace{\begin{pmatrix}
					\mathbf{J}_\delta^{-1}\mathbf{B}
					& \mathbf{0} \\
					\mathbf{0}  &
					\mathbf{I}_2
			\end{pmatrix}}_\mathbf{G}
			\mathbf{n}_{\rm imu}.
		\end{equation}
		The above approximation is unrealistic 
		in the sense that 
		we cannot construct an estimate
		$\mathbf X_I$ in any way in practice to generate such $\xi_{I}$ evolving like the right-hand side of~\eqref{eqn:correct_error_propagation}. However, we can
		work out a better estimation error covariance using~\eqref{eqn:correct_error_propagation}.
		Let  $\mathbf{P}(t):=\mathop{\mathbb{E}}
		[(\xi_I~\tilde{\mathbf{b}})(\xi_I~\tilde{\mathbf{b}})^\top]$
		denote the estimate error covariance. In virtue of~\eqref{eqn:correct_error_propagation}, we use the following equation to approximate the evolution of $\mathbf P(t)$:
		\begin{equation}\label{eqn:compute_propagated_covariance}
			\dot{\mathbf{P}}=\mathbf{F}\mathbf{P}+\mathbf{P}\mathbf{F}^\top+\mathbf{G}\mathbf{Q}_{{\rm imu}}\mathbf{G}^\top
		\end{equation}
		with $\mathbf F$ and $\mathbf G$ given in~\eqref{eqn:correct_error_propagation}.
		Compared to the exact error covariance propagation, numerical integration under~\eqref{eqn:compute_propagated_covariance} is more viable in terms of computation efficiency to predict the error covariance to the next measurement sampling time since
		$\xi_\delta$ can be independently drawn beforehand.
		Note that the estimation error is seldom zero most of the time and sometime may be large due to external disturbance in the experiment. Hence the ``trick'' that uses $\mathbf F$ and $\mathbf G$ in~\eqref{eqn:correct_error_propagation} to propagate the covariance works better than the linearization approximation technique in most cases. 
		Regarding the problem of  choosing a good distribution for $\xi_\delta$, 
		our experimental experience is reported in Section~\ref{epxeriment_evaluatoin:silumation}. Our current understanding of this issue is far from sufficient, and we may need further investigation on it in future work.
	\end{enumerate}

	\subsection{Logarithmic Invariant Error Propagation in VINS }
	In this part, we apply the theoretical results to filter design for motion state estimation of a mobile robot with the IMU in a real-time VINS application. Now suppose there are $n$ landmarks, the positions of which denoted as $\mathbf{f}=(\mathbf{f}_1\cdots\mathbf{f}_n)$.
	In VINS, we need to estimate the robot's real-time attitude and position as well as the positions of the landmarks.
	The system state, including the robot's pose and the landmarks' position, should be augmented to an element in $SE_{n+2}(3)$, that is,
	$$\begin{pmatrix}
		\begin{matrix}
			\mathbf{R}
		\end{matrix}
		& \rvline & \mathbf{p}\quad\mathbf{v}\quad\mathbf{f} \\
		\hline
		\mathbf{0}_{3\times3} & \rvline &
		\begin{matrix}
			\mathbf{I}_{3}
		\end{matrix}
	\end{pmatrix}$$ 
	and the right invariant error for an estimate can be expressed as  
	$$ \begin{pmatrix}\begin{matrix}
			\tilde{\mathbf{R}}
		\end{matrix}
		& \rvline & \hat{\mathbf{p}}-\tilde{\mathbf{R}}\mathbf{p}\quad\hat{\mathbf{v}}-\tilde{\mathbf{R}}\mathbf{v} \quad\hat{\mathbf{f}}-\tilde{\mathbf{R}}\mathbf{f}\\
		\hline
		\mathbf{0}_{3\times3} & \rvline &
		\mathbf{I}_3
	\end{pmatrix}.$$
	In this paper we consider spatially static landmarks with $\dot{\mathbf{f}} = 0$.
	Let $\xi_{X}$ denote the logarithm invariant error. Since the system dynamics satisfies the form of~\eqref{eqn:dynamic_sen(3)}, by Theorem~\ref{thm: log_linear}, we obtain
	\begin{equation}\label{eqn:vins_propogation}
		\begin{split}
			&\begin{pmatrix}
				\dot{\xi}_{X}\\
				\dot{\tilde{\mathbf{b}}}
			\end{pmatrix}=
			\underbrace{ \begin{pmatrix}
					\mathbf{A}_f
					& -\mathbf{B}_f \\
					\mathbf{0}  &
					\mathbf{0}
			\end{pmatrix}}_\mathbf{F}
			\underbrace{\begin{pmatrix}
					\xi_{X}\\
					\tilde{\mathbf{b}}
			\end{pmatrix}}_\mathbf{e}+\underbrace{\begin{pmatrix}
					\mathbf{B}_f
					& \mathbf{0} \\
					\mathbf{0}  &
					\mathbf{I}_2
			\end{pmatrix}}_\mathbf{G}
			\mathbf{n}_{\rm imu}\\
			&\mathbf{A}_f = \begin{pmatrix}
				\mathbf{0}&\mathbf{0}&\mathbf{0}&\mathbf{0}\\
				\mathbf{0}&\mathbf{0}&\mathbf{I}_{3}&\mathbf{0}\\
				\mathbf{g}^\wedge&\mathbf{0}&\mathbf{0}&\mathbf{0}\\
				\mathbf{0}&\mathbf{0}&\mathbf{0}&\mathbf{0}
			\end{pmatrix}
			\quad\mathbf{B}_f=\mathbf{J}(\ad_{{\xi_{X}}^\wedge})^{-1}\begin{pmatrix}
				\hat{\mathbf{R}}&\mathbf{0}\\
				\hat{\mathbf{p}}^\wedge\hat{\mathbf{R}}&\mathbf{0}\\
				\hat{\mathbf{v}}^\wedge\hat{\mathbf{R}}&\hat{\mathbf{R}}\\
				\hat{\mathbf{f}}^\wedge\hat{\mathbf{R}}&\mathbf{0}\end{pmatrix}.
		\end{split}
	\end{equation}
	The approximation methods provided in~Section~\ref{section:Approximation_methods} can also be applied to approximate $\mathbf{B}_f$.
}

\subsection{Visual Measurements Update in VINS}
In the VINS problem, external feature positions measurements should be obtained by the camera. The camera's pose relative to the world frame can be expressed in terms of the IMU's pose and position as well as the relative pose and position between the camera and IMU:
\begin{equation*}
	\mathbf{X}_{C} = \begin{pmatrix}
		\mathbf{R}
		& \mathbf{p} \\
		{0} &
		1
	\end{pmatrix}\mathbf{X}_{I}^{C},
\end{equation*}
where $\mathbf{X}_{I}^{C}:=\begin{pmatrix}
	\mathbf{R}_{I}^{C}
	& \mathbf{p}_{I}^{C} \\
	{0} &  
	1
\end{pmatrix}\in SE(3)$ encodes the relative rotation $\mathbf{R}_{I}^{C}$ from
the IMU frame to the camera one as well as the
relative
position $\mathbf{p}_{I}^{C}$ of the origin of the camera frame  relative to the IMU frame.
In practice,  $\mathbf{R}_{I}^{C}$ and  $\mathbf{p}_{I}^{C}$ may be accurately obtained through hardware calibration. 

When the camera is exploring the environment and tracking the landmark, the measurement model in the discrete-time form at time step $k$~(the continuous-time form denoted as $t_k$) is defined as follows:
\begin{equation}\label{eqn:right_invariant_error_measurement}
	\mathbf{z}(k)=\mathbf{\pi}\left( \mathbf{X}_{C}(k)^{-1}\mathbf{f} \right) + \mathbf{n}_{C }(k)\quad \mathbf{n}_{C }\sim\mathcal{N}(0,\mathbf{N})
\end{equation} 
where $\pi(\mathbf x):=\mathbf K\frac{\mathbf x}{\|\mathbf x\|}$ denotes the camera projection mapping with $\mathbf K$ being the camera intrinsic matrix, and $\mathbf{n}_{C  }$ is the measurement noise. The linearized measurement equation w.r.t. the errors writes:
\begin{equation}
	\begin{split}
		&\tilde{\mathbf{z}}=\mathbf{z}(k+1) -\mathbf{\pi}\left(\hat{\mathbf{X}}_{C}(k+1|k)^{-1}\hat{\mathbf{f}} \right)=	\mathbf{H}{\mathbf{e}}+O(\left\|\mathbf{e} \right\|^2 )\\
		&\mathbf{H}=	 \partial\pi\cdot\mathbf{R}_{I}^{C_i-1}\mathbf{\hat{R}}^{-1}\begin{pmatrix}
			\mathbf{0}_3&-\mathbf{I}_3&\mathbf{0}_3&\mathbf{0}_3&\mathbf{0}_3&\mathbf{I}_3
		\end{pmatrix}
	\end{split}
\end{equation}
where $\partial \pi$ represents the projection Jacobian and $\hat{\mathbf{X}}_C(k+1|k)$ denotes \textit{a priori} estimate state.

The IEKF update proceeds by calculating the Kalman gian:
\begin{equation}
	\mathbf{K}=\mathbf{P}(k+1|k)\mathbf{H}^{\top}\left(\mathbf{H}\mathbf{P}(k+1|k)\mathbf{H}^{\top}+\mathbf{N} \right)^{-1}.
\end{equation} 
After obtaining the Kalman gain and the transformed residuals, the estimated state are corrected separately. Because of the right-invariant observation equation~\eqref{eqn:right_invariant_error_measurement}, the innovation also depends on the invariant error and the updated equation for the IMU pose $\mathbf{X}_I$ takes the following form~\cite{Barrau}:
\begin{equation}\label{eqn:update_state}
	\begin{split}
		\hat{\mathbf{X}}_{I}(k+1|k+1)&=\exp{\left( (\mathbf{K}\tilde{\mathbf{z}})_{[I]}\right) }\hat{\mathbf{X}}_{I}(k+1|k) ,\\
	\end{split}
\end{equation}
where $(y)_{[x]}$ represents the row  of the vector $y$ corresponding with $x$. Because the bias cannot be added in the matrix group, their updated equation writes as follows 
\begin{equation}\label{eqn:eqn:update_bias}
	\hat{\mathbf{b}}(k+1|k+1) = \hat{\mathbf{b}}(k+1|k) +(\mathbf{K}\tilde{z})_{[\mathbf{\hat{b}}]}.
\end{equation}

The corresponding covariance matrix is updated by:
\begin{equation}\label{eqn:update_covariance}		\mathbf{P}(k+1|k+1)=\left(\mathbf{I}-\mathbf{K}\mathbf{H}_{\mathbf{X}}\right)\mathbf{P}(k+1|k).
\end{equation}

{
	
	
	\subsection{Discussion on the Estimators' Consistency}
	The observability of a dynamical system, which indicates whether we are capable of recovering the initial states of the system with a sequence of system output of a certain length, can be used for consistency analysis of an estimator design for the system~\cite{Huang2008}. We will discuss the consistency of the proposed estimator via observability, following the work \cite{Huang2008}. Note that irrespective of the estimate of  $\mathbf{b}$ in~\eqref{eqn:vins_propogation}, our logarithmic invariant error propagation is linear, time-invariant (To be precise, the linear dynamics is nilpotent.). In the discretized dynamical system for a given  sampling period $\delta t$, the transition matrix is a matrix of  polynomials in $\delta t$ as follows:
	\begin{equation*}
		\Phi=\begin{pmatrix}
			\mathbf{I}&\mathbf{0}&\mathbf{0}&\mathbf{0}\\
			\frac{1}{2}\mathbf{g}^\wedge\delta t^2&\mathbf{I}&\mathbf{I}\delta t&\mathbf{0}\\
			\mathbf{g}^\wedge\delta t&\mathbf{0}&\mathbf{I}&\mathbf{0}\\
			\mathbf{0}&\mathbf{0}&\mathbf{0}&\mathbf{I}
		\end{pmatrix}.
	\end{equation*}
	The observation matrix $\mathcal{O}$ is given by
	\begin{equation*}
		\mathcal{O}=\begin{pmatrix}
			\mathbf{H}\\
			\mathbf{H}\Phi\\
			\mathbf{H}\Phi^2\\
			\vdots
		\end{pmatrix}=\begin{pmatrix}
			\mathbf{0}&-\mathbf{I}&\mathbf{0}&\mathbf{I}\\-\frac{1}{2}\mathbf{g}^\wedge\delta t^2&-\mathbf{I}&-\mathbf{I}\delta t&\mathbf{I}\\-\frac{3}{2}\mathbf{g}^\wedge\delta t^2&-\mathbf{I}&-2\mathbf{I}\delta t^2&\mathbf{I}\\\vdots&\vdots&\vdots&\vdots
		\end{pmatrix}.
	\end{equation*}
	The six columns of the matrix are linearly dependent, which is consistent with the nonlinear observability result of the original system~\cite{Huang2008}. 
	Notice that the introduction of the imitative logarithmic invariant error $\xi_{\delta}$
	in place of the non-accessible $\xi_X$ in an estimator has nothing to do with the system observability. Hence, the linear approximation and the imitative invertible Jacobians design for the VINS maintain consistency. 
}

	\subsection{Visual Measurement Update in MSCKF}
	In the conventional VINS problem, enormous external features positions need to be estimated and included in the filter's state vector, which leads to high computational cost. To tackle this problem, in the multi-state constraints Kalman filter (MSCKF)~\cite{Mourikis2007}, the filtering frame does not require putting the 3D feature position but the camera's multi-state poses into the state vector. Besides, this frame increases the robustness against invariant error approximation and linearization of the camera measurement equation. This subsection will introduce the IEKF method embedded in MSCKF.
	
	The MSCKF maintains a window of sequential camera poses and augments the state via discarding the past camera poses and coping with the current one every time a new image comes. Suppose that $N$ consecutive camera poses states are cloned and stacked into the state vector. In addition, the vector also includes the augmented IMU state at the current time. 
	\begin{equation*}
		\mathbf{X}(t)=(\mathbf{X}_{I}(t),~ \mathbf{b}(t),~\mathbf{X}_{C1},\ldots,\mathbf{X}_{CN})	
	\end{equation*} 
	where $\mathbf{X}_{Ci}\in SE(3), i=1,\ldots,N$, denote the pose estimates of the camera.
	
	Computing the covariance at propogation step is the same as the conventional VINS. Regarding the update step, when the feature is captured at the $k$th sampling time by the $i$th camera pose, the measurement model is defined as follows:
	\begin{equation*}
		\mathbf{z}_i(k)=\mathbf{\pi}\left( \mathbf{X}_{C_{i}}(k)^{-1}\mathbf{f} \right) + \mathbf{n}_{C_{i}}(k)
	\end{equation*} 
	where $\pi(\mathbf x):=\mathbf K\frac{\mathbf x}{\|\mathbf x\|}$ denotes the camera projection mapping with $\mathbf K$ being the camera intrinsic matrix, $\mathbf{f}$ is the position of the feature in the world frame and $\mathbf{n}_{C_{i}}$ is the measurement noise, which is Gaussian with covariance $\mathbf{N}_i$. By solving a least-squares minimization problem as a first step\cite{Mourikis2007}, the feature position can be estimated and denoted by $\mathbf{\hat{f}}$. Given $\mathbf{\hat{f}} $, we compute the residual term for the $i$th image frame as follows:
	\begin{equation*}
		\tilde{\mathbf{z}}_i(k)=\mathbf{z}_i(k)-\mathbf{\hat{z}}_i(k).
	\end{equation*}
	This residual term is then linearized with respect to ${\xi}_{Ci}$ and the feature estimated error $\mathbf {\tilde{f}}:=\mathbf{\hat{f}}-\mathbf{f}$:
	\begin{equation*}\label{eqn:compute_H}
		\begin{split}
			&\tilde{\mathbf{z}}_i(k+1)=\partial\pi\cdot \left( \mathbf{H}_{\mathbf{X}_{Ci}}\xi_{C_{i}}(k+1|k)+\mathbf{H}_{\mathbf{f}}\tilde{\mathbf{f}}\right) +\mathbf{n}_C(k+1) \\
			&\mathbf{H}_{\mathbf{X}_{Ci}}  =\left(\mathbf{R}_{I}^{C_i-1}\mathbf{\hat{R}}(k+1|k)^{-1}\mathbf{f}^\wedge\quad -\mathbf{R}_{I}^{C_i-1}\mathbf{\hat{R}}(k+1|k)^{-1}\right)\\
			&\mathbf{H}_{f} =\mathbf{R}_{I}^{C_i-1}\mathbf{\hat{R}}(k+1|k)^{-1}.\\
		\end{split}
	\end{equation*}
	To overcome the problem that the estimated landmark errors $\tilde{\mathbf{f}}$ is correlated to the estimated errors $\xi_{Ci}$, we define a transformed residue $\tilde{\mathbf{z}}_{oi}$ by projecting $\tilde{\mathbf{z}}_{i}$ onto the left nullspace of the matrix $\mathbf{H}_f$. The projection is carried out by a QR decomposition:
	\begin{align}
		\mathbf{H}_{f} &= \begin{pmatrix} \mathbf{Q_1} & \mathbf{Q_2} \end{pmatrix} \begin{pmatrix} \mathbf{R_1} \\ \mathbf{0} \end{pmatrix} = \mathbf{Q_1}\mathbf{R_1}\\
		\begin{split}\label{eqn:compute H_xo}
			\mathbf{Q_2}^\top\tilde{\mathbf{z}}_{i} &= \mathbf{Q_2}^\top\mathbf{H}_{\mathbf{X}_{Ci}} \mathbf{\xi}_{Ci} + \mathbf{Q_2}^\top\mathbf{n}_{Ci} \\[5px] \empty \Rightarrow~ \tilde{\mathbf{z}}_{oi} &= \mathbf{H}_{\mathbf{X}_{Cio}}\mathbf{\xi}_{Ci} + \mathbf{n}_{o,Ci}. 
		\end{split}
	\end{align}
	By stacking all transformed residuals in one vector $\mathbf{z}_o$, we obtain that 
	\begin{equation}
		\tilde{\mathbf{z}}_{o}= \mathbf{H}_{\mathbf{X}_{o}}\xi+\mathbf{n}_{o}
	\end{equation}
	where the covariacne of $\mathbf{n}_{o}$ is denoted by $\mathbf{N}$.

\section{Experimental Evaluation} \label{sec:evaluation}
We evaluate the proposed algorithm based on the closed-form expression of right invariant error propagation by simulations and experiments on datasets. We evaluate the consistency of our algorithms and standard extended Kalman filter by simulations. We compare the imitating Jacobians for Invariant Extended Kalman Filter (IJ-IEKF) with quaternion-based extended Kalman filter (QEKF)~\cite{Trawny2005IndirectKF},  first estimates Jacobian EKF (FEJ)~\cite{Huang2008}, and IEKF in experiments. 
{
	\subsection{Simulations}\label{epxeriment_evaluatoin:silumation}
	We simulate the IMU readings and camera measurements with the trajectory~shown in Fig.~\ref{fig:trajectory} and implement 50 Monte-Carlo runs with the EKF, IEKF, and IJ-IEKF algorithms. In this simulation, we use uniform distributed random variables $\xi_{\delta}\sim\mathcal{U}(-r,r)$ 
	where $r$ denotes the range of invariant error and can be set heuristically from experimental experience. We then calculate $\mathbf{J}_\delta:=\mathbf{J}(\ad_{\xi_\delta^\wedge})$ via Definition~\ref{def:Jacobians}. Different ranges~($r=0.01,0.1,0.5,1.0,2.0$) are set in the simulation to imitate the Jacobians of orientation errors. It is noted that the position errors to imitate are set as $0$. The root mean squared error~(RMSE) and the normalized estimation error squared~(NEES), which has been divided by the degrees of freedoms of the state variables, are used for evaluation. Results are shown in Table~\ref{tab:Simulation_result} and Fig.~\ref{fig:Simulation}.
	\begin{table}[h]
		\centering
		\caption{\textbf{Simulations.} Root Mean Squared Error~(RMSE) and Normalized Estimation Error Squared (NEES)}
		\label{tab:Simulation_result}  
		
		\begin{tabular}{c|ll|ll}
			\hline\hline\noalign{\smallskip}
			& \multicolumn{2}{|c|}{\textbf{RMSE}} & \multicolumn{2}{|c}{\textbf{NEES}} \\ \hline
			& \textbf{Pos.(m)} & \textbf{Ang.(rad)} &\textbf{Pos.} & \textbf{Ang.} \\ \hline
			\textbf{EKF}  & 1.1520&0.0198&2.9966&2.5539\\ \hline
			\textbf{IEKF} &  0.6916 & 0.0147& 1.1411 & {1.1221}     \\ \hline
			\textbf{IJIEKF-0.01} & 0.6871 &\textbf{ 0.0146}& 1.1218 & 1.1094 \\ \hline
			\textbf{IJIEKF-0.1}  & \textbf{0.6839} & 0.0150& 1.0938 & 1.1389  \\ \hline
			\textbf{IJIEKF-0.5}  &0.7548&0.0174&0.9815&1.0621\\ \hline
			\textbf{IJIEKF-1.0}  &0.8481&0.0161&0.9275&0.5528 \\ \hline
			\textbf{IJIEKF-2.0}  &1.1619&0.0254&0.9871&0.5115 \\ \hline
		\end{tabular}
	\end{table}
	
	\begin{figure*}
		\subfigure[Trajectory.]{ 
			\centering
			\includegraphics[width=0.3\textwidth]{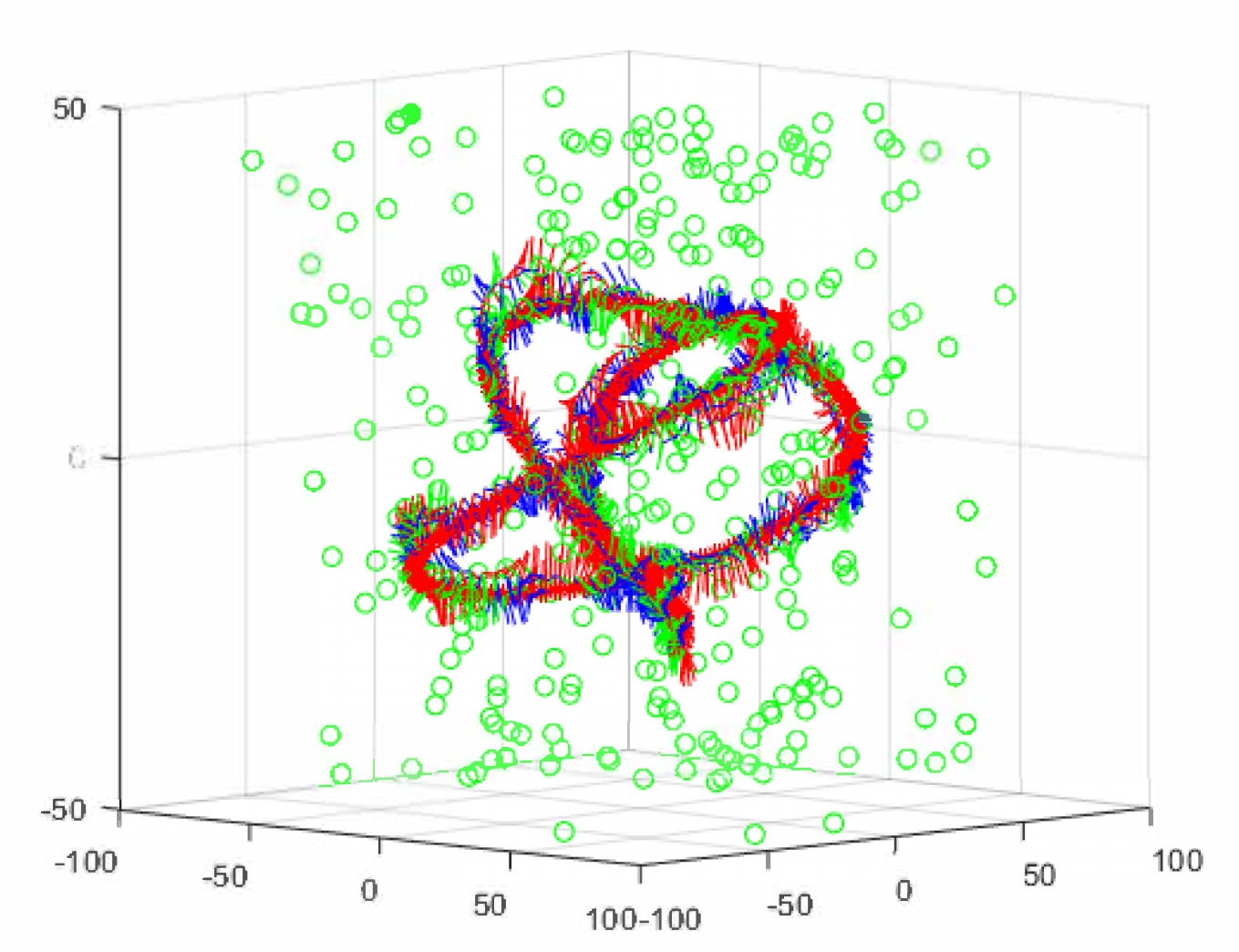} 
			\label{fig:trajectory} 
		}%
		\subfigure[Root Mean Squared Error~(RMSE).]{
			\centering
			\includegraphics[width=0.3\textwidth]{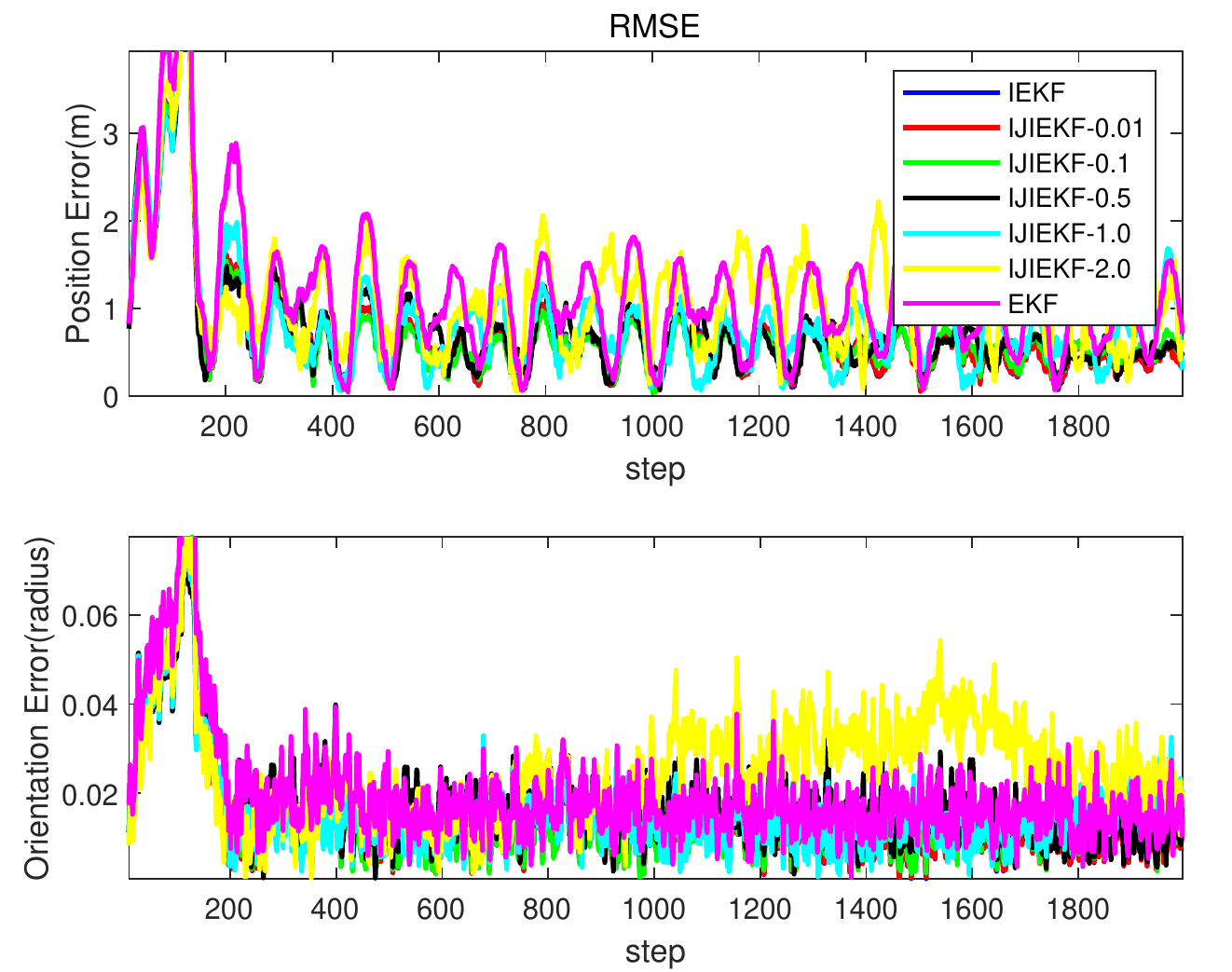}
			\label{fig:rmseSim}
		}%
		\subfigure[Normalized Estimation Error Squared.]{
			\centering
			\includegraphics[width=0.287\textwidth]{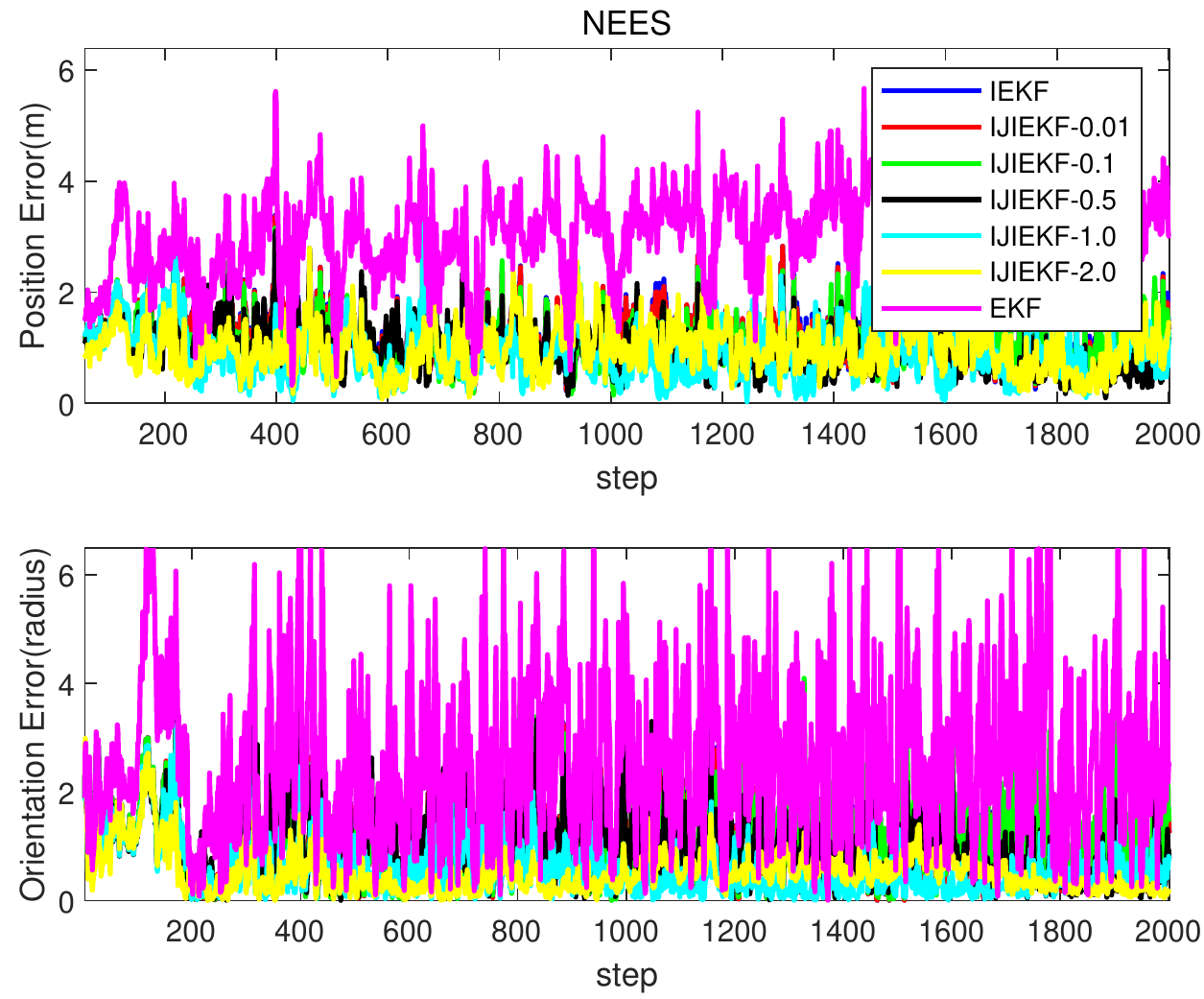}
			\label{fig:neesSim}
		}
		\caption{The simulation results for 50 Monte-Carlo runs of \textbf{EKF}, \textbf{IEKF} and \textbf{IJ-IEKF} algorithms with $\xi_\delta\sim \mathcal{U}(-r,r)$, and we set $r=0.01,0.1,0.5,1.0,2.0$. Fig.~\ref{fig:trajectory} shows the simulated trajectory with $\begin{pmatrix}
				50\cos{(0.075t)}&40\sin{(0.05t)}&20\sin{(0.05t+1)}
			\end{pmatrix}^\top$. Fig.~ \ref{fig:rmseSim} and Fig.~\ref{fig:neesSim} show RMSE and NEES of the above-mentioned algorithms at every time step during experiments.}
		\label{fig:Simulation}
	\end{figure*}	
	As shown in Fig.~\ref{fig:neesSim}, the IJ-IEKF and IEKF maintain consistency with different ranges $r$, while the EKF is not consistent. Table.~\ref{tab:Simulation_result} illustrates that imitating the invariant error to compensate the Jacobian term improves the filter performance, which is also demonstrated in the following experiments.}

\subsection{Experiments on Datasets}
We further evaluate the accuracy of our proposed algorithms by testing them on the EuRoC MAV Datasets \cite{Burri25012016}, which are collected on board a macro aerial vehicle in indoor environments. All the experiments are implemented on the OPENVINS platform~\cite{Geneva2020ICRA} with an Intel i7-11700 processor. We maintain the multi-state constraint Kalman filter~(MSCKF)~\cite{Mourikis2007} frame in the OPENVINS for its efficiency, where online calibration and loop closure modules are turned off.  			

We run this experiment with QEKF, FEJ, IEKF, and our proposed IJ-IEKF. We set the imitating orientation satisfies $\mathcal{U}(-0.5,0.5)$. The other experimental parameters settings are shown in Table~\ref{tab:Simulation_and_Experimental parameters}. We select all \textbf{MH room} sequences~(including easy, medium, and difficult mode) in the  EuRoC MAV Datasets to demonstrate the performance of the above-mentioned algorithms. We choose the relative pose error (RPE) and the absolute trajectory error (ATE) as evaluation metrics. In order to analyze the performance on a single dataset, we also choose the RMSE as another evaluation metric. The reader can refer to~\cite{Zichao2018} to find more details on the evaluation metrics. In the whole experiment, it costs less than $10$ms to perform the propagation step and the update step in all filter-based algorithms with about $20$ captured features each step.

\begin{table}[h]
	\centering
	\caption{Experiment Parameters}
	\label{tab:Simulation_and_Experimental parameters}  
	\resizebox{0.48\textwidth}{!}{
		\begin{tabular}{cccc}
			\hline\hline\noalign{\smallskip}	
			Parameters& Value& Parameters&Value\\
			\noalign{\smallskip}\hline\noalign{\smallskip}
			Gyro White Noise$({\rm rad/(s\sqrt{HZ})})$& 1.6968e-04&Gyro. Random Walk$({\rm rad/(s^2\sqrt{HZ})})$ & 1.9393e-05  \\
			Accel. White Noise$({\rm m/(s^2\sqrt{HZ})})$ & 2.0000e-3 &Accel. Random Walk $({\rm m/(s^3\sqrt{HZ})})$ & 3.0000e-3  \\
			Max. Clone Size&11&Max Feats&40\\
			\noalign{\smallskip}\hline
	\end{tabular}}
\end{table}

The boxplots in Fig.~\ref{fig:dataset} illustrate that the error distribution of IJ-IEKF is closer to $0$ and more concentrated, especially in terms of relative orientation error. Moreover, we choose a representative challenging sequence \textbf{MH\_04\_difficult} to show the difference in performance among QEKF, FEJ, IEKF, and IJ-IEKF, as shown in Fig.~\ref{fig:errors_rmse}. It illustrates that the odometry estimates from IJ-IEKF contain minimal drift to keep the estimated trajectory close to the ground truth most of the time during the experiment, especially on the orientation error at the end of the sequence. All the evaluations on the ATE metrics are summarized in Table~\ref{tab:Experimental_results}. 

\begin{figure}
	\centering
	\includegraphics[width=0.45\textwidth]{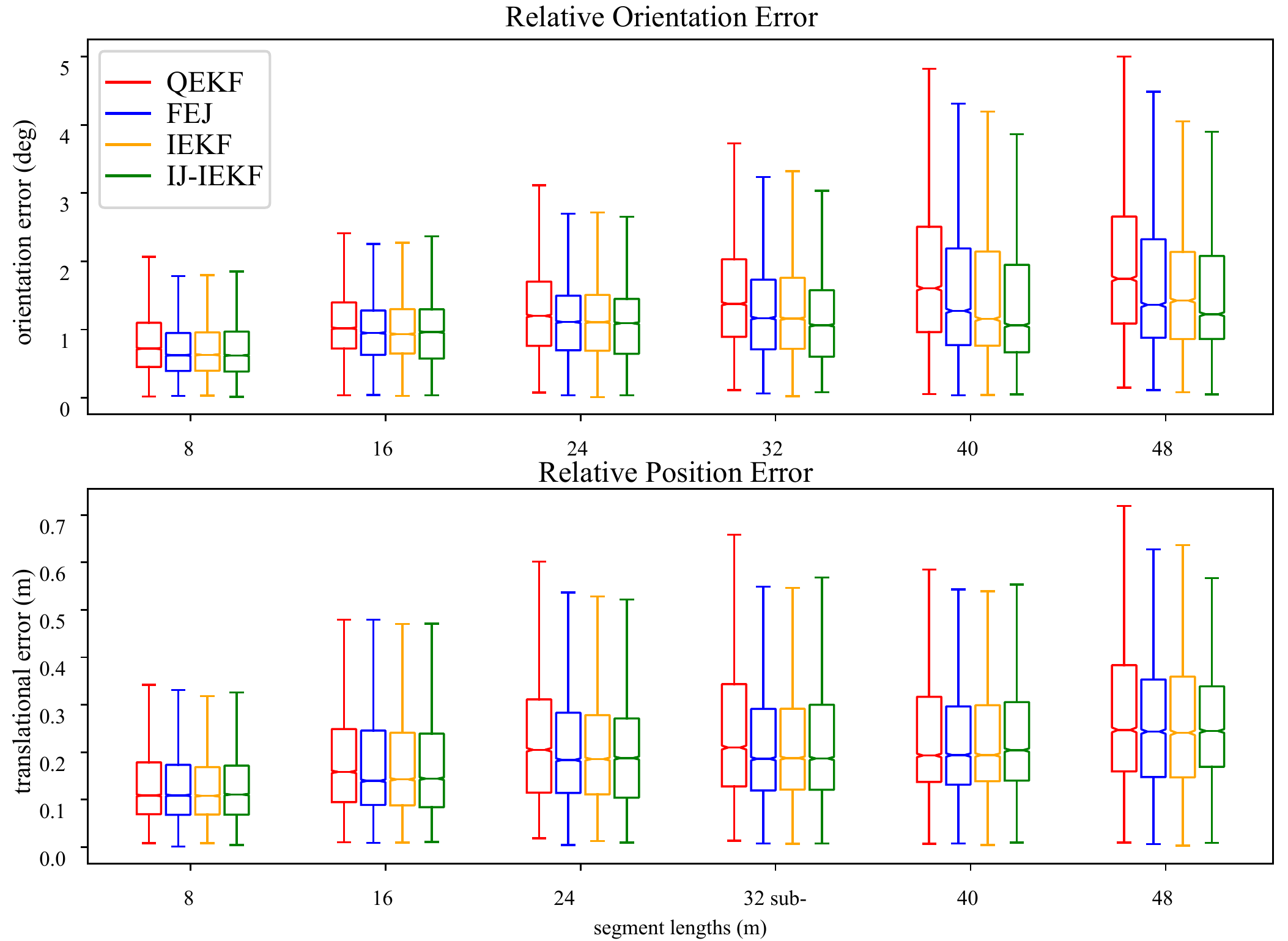}
	\caption{Boxplots of the odometric relative pose error using \textbf{QEKF}, \textbf{FEJ}, \textbf{IEKF} and \textbf{IJ-IEKF} algorithms on different segments of easy, medium, and difficult sequences in the EuRoC MAV dataset.}\label{fig:dataset}
\end{figure}
\begin{figure}
	\centering
	\includegraphics[width=0.45\textwidth]{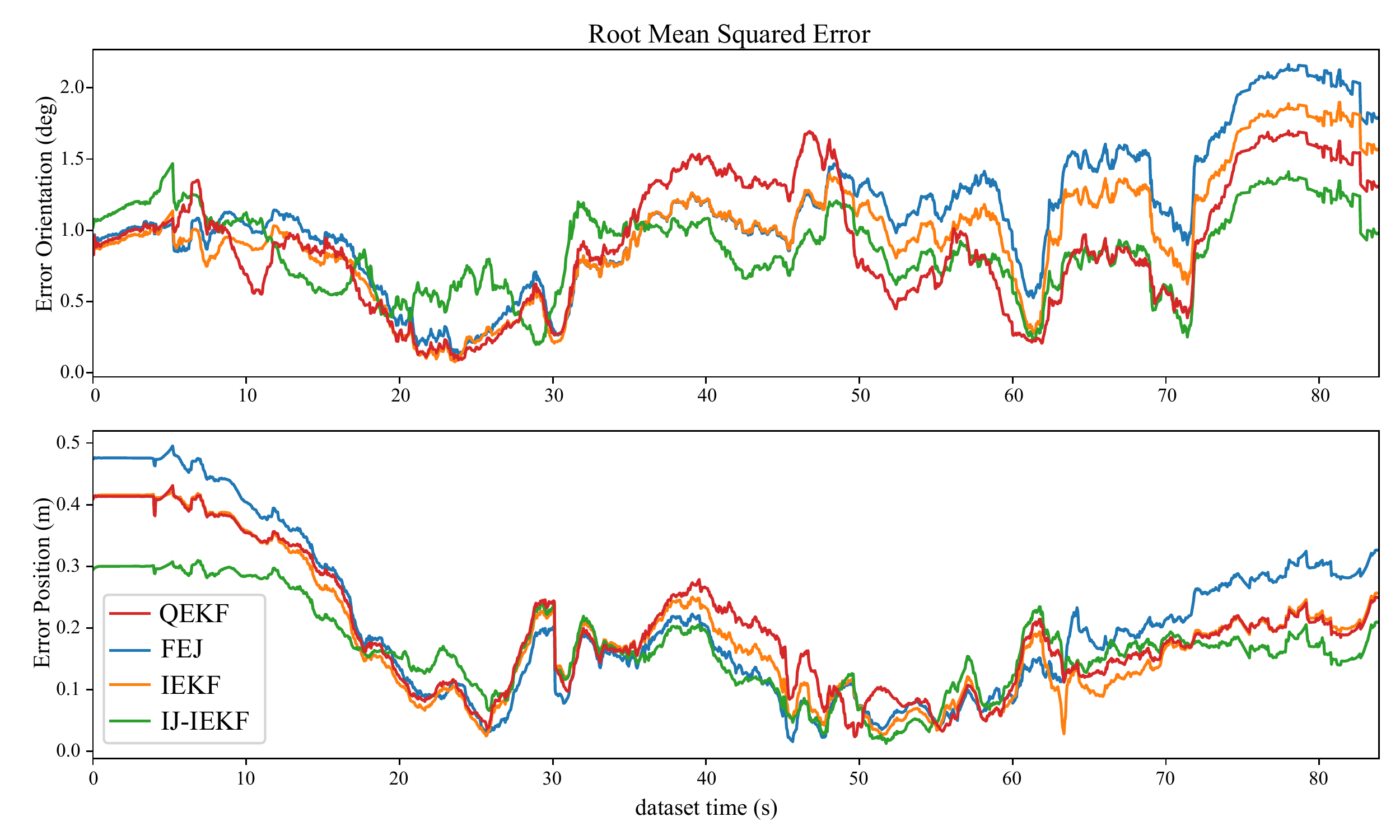}
	\caption{The root mean squared errors evolution for the \textbf{QEKF}, \textbf{FEJ}, \textbf{IEKF} and \textbf{IJ-IEKF} algorithms on the challenging sequence \textbf{MH\_04\_difficult}. }\label{fig:errors_rmse}
\end{figure}
\begin{table*}
	\centering
	\caption{\textbf{Experiments.} Absolute Trajectory Error(ATE)}
	\label{tab:Experimental_results}  
	\begin{tabular}{ccccccc}
		\hline\hline\noalign{\smallskip}	
		($^\circ$)/($m$)  & \textbf{MH\_01\_easy} & \textbf{MH\_02\_easy} & \textbf{MH\_03\_medium} & \textbf{MH\_04\_difficult} & \textbf{MH\_05\_difficult}& \textbf{Average} \\\hline
		\textbf{QEKF} & 2.415 / 0.155 & 0.940 / 0.089 & 1.737 / 0.109 & 1.003 / 0.222 & 1.228 / 0.266& 1.465 / 0.168 \\
		\textbf{FEJ} & 1.925 / 0.120 & 0.745 /\textbf{ 0.104} & 1.375 /\textbf{ 0.091 }& 1.201 / 0.245 &\textbf{ 0.877} / 0.250 & 1.225 / 0.162 \\
		\textbf{IEKF} & 1.930 / 0.120 & 0.763 / 0.098 & 1.281 / 0.096 & 1.062 / 0.215 & 0.916 / 0.250 & 1.191 / 0.156 \\
		\textbf{IJ-IEKF} & \textbf{1.890 }/ \textbf{0.113 }& \textbf{0.723 }/ 0.108 &\textbf{ 1.212 }/ 0.094 & \textbf{0.909 / 0.183} & 0.945 /\textbf{ 0.238} & \textbf{1.136 }/ \textbf{0.147} \\
		\noalign{\smallskip}\hline
	\end{tabular}
\end{table*}
\section{Conclusions and Future Work}\label{sec:conclusion and future work}
In this paper, we demonstrated how the noisy nonlinear term influences the evolution of  estimation error through theoretical analysis and experimental tests. We analyzed the invariant error on $SE_n(3)$ and derived a closed-form expression for the propagation of the Lie logarithm of the invariant error on $\mathfrak{se}_n(3)$. We applied our theoretical findings to 
the navigation model of IMU, and then the IMU model was implemented in VINS. The experimental results showed that the filter based on compensating the nonlinear terms seems likely to perform better.  The closed-form expression of uncertainty propagation for the estimation error in the presence of noise is worth further investigation for optimal filtering design.

	{\appendix[Supporting Lemmas]
	\begin{lemma}[Theorem 5 of~\cite{Hunacek2008}]\label{lemma:derivative_of_lie_group}
		For a continuously differentiable trajectory $\mathbf{X}(t)\in\mathfrak{g}$. The time-derivative of $e^{\mathbf{X}}$ has the following closed form:
		\begin{equation*}\label{eqn:diff_between_group_and_algebra}
			{\frac {\rm d}{\rm dt}}e^{\mathbf{X}(t)}=e^{\mathbf{X}(t)}{\frac {\mathbf{I}-e^{-\mathrm {ad} _{\mathbf{X}}}}{\mathrm {ad} _{\mathbf{X}}}}{\dot {\mathbf X}(t)}.
		\end{equation*}
	\end{lemma}

	\begin{lemma}\label{lemma:Adjoit_representation_of_B}
		Given any  $\mathbf{A}, \mathbf{B}\in\mathfrak{g}$, it holds that 
		\begin{equation*}\label{eqn:Adjoit_representation_of_B}
			e^{ \mathbf{A} }\mathbf{B}e^{ -\mathbf{A}}=e^{\ad_{\mathbf{A}} }\mathbf{B}.
		\end{equation*}
	\end{lemma}
	\begin{proof}	
		First we define the following mapping: for $x\in\mathbb R$,
		\begin{equation*}
			h(x)=\exp{\left(x\mathbf{A} \right) }\mathbf{B}\exp{\left(-x\mathbf{A} \right) }.
		\end{equation*}	
		which admits a Taylor series at $x_{0}=0$ as
		$h(x)=\sum_{n=0}^{\infty}\frac{1}{n!}h_n x^{n}$,
		where $h_n$ are matrix Taylor coefficients of appropriate dimensions.  Evaluating the above series at $x=0$,  we have $h_0=\mathbf B$.

		Taking derivative of $h(x)$ with respect to $x$, it yields that
		\begin{align*}
			\frac{\rm d}{\rm dx}h(x)&=\frac{\rm d e^{x\mathbf A}}{\rm dx} \mathbf B e^{(-x\mathbf A)}+e^{(x\mathbf A)}\mathbf B     \frac{\rm d e^{(-x\mathbf A)}}{\rm dx} \\
			&=\mathbf Ae^{\left(x\mathbf{A} \right) }\mathbf{B}e^{\left(-x\mathbf{A} \right) }-e^{\left(x\mathbf{A} \right) }\mathbf{B}e^{\left(-x\mathbf{A} \right) }\mathbf A\\
			&=\left[\mathbf A, e^{\left(x\mathbf{A} \right) }\mathbf{B}e^{\left(-x\mathbf{A} \right) }\right]\\
			&= [\mathbf A, h(x)]\\
			&= \sum_{n=0}^{\infty}\frac{1}{n!}\left[\mathbf A,  h_n\right]  x^{n},
		\end{align*} 
		where the last equality holds due to the continuity of the Lie bracket. On the other hand, we know that for any $x\in\mathbb R$,
		\begin{equation*}
			\sum_{n=0}^{\infty}\frac{1}{n!}h_{n+1} x^{n}=
			\frac{\rm d}{\rm dx}h(x)=
			\sum_{n=0}^{\infty}\frac{1}{n!}\left[\mathbf A,  h_n\right]  x^{n}.
		\end{equation*}
		Evaluating the above equation at $x=0$, it leads to
		$h_1=[\mathbf A, h_0]=[\mathbf A, \mathbf B]$. Keeping taking higher-order derivatives of $h(x)$, it leads to that $h_{n+1}=
		\left[\mathbf A,  h_n\right]$ for $n=0,1,\ldots$. Finally, we conclude the proof by setting $x=1$ for 
		$h(x)$.
	\end{proof}
	
	\begin{lemma}\label{lemma:invertible matrix}
		For $\xi\in\mathfrak{g}$, $\mathbf J(\mathrm{ad}_{\xi})=\frac {1-e^{-\mathrm {ad} _{\xi}}}{\mathrm {ad} _{\xi}}$ in matrix form is invertible  when the eigenvalues
		$ \lambda_{i}$'s of $\ad_{\xi}$ satisfy $\lambda _{i}\neq 2k\pi {\mathbf i},k=\pm 1,\pm 2,\ldots$, where $\mathbf i$ is the imaginary unit with $\mathbf i^2=-1$.
	\end{lemma}
	\begin{proof}
		The matrix $\mathbf {J}(\mathrm{ad}_{\xi})=\frac {1-e^{-\mathrm {ad} _{\xi}}}{\mathrm {ad} _{\xi}}$ can be expanded into a Taylor series as 
		$\mathbf {J}(\mathrm{ad}_{\xi})=\frac {1-e^{-\mathrm {ad} _{\xi}}}{\mathrm {ad} _{\xi}}=\sum_{n=0}^{\infty}
		\frac{1}{(n+1)!}{\ad_{\xi}}^n$.
		Since $\lambda_{i}$'s are the eigenvalues of $\ad_{\xi}$, i.e., there exists $\mathbf{x}\neq\mathbf{0}$ such that $\ad_{\xi}\mathbf{x}=\lambda_{i}\mathbf{x}$, we have that 
		\begin{equation}\label{eqn:series_eigen}
			\sum_{n=0}^{k}
			\frac{1}{(n+1)!}{\ad_{\xi}}^n \mathbf x =
			\sum_{n=0}^{k}
			\frac{1}{(n+1)!}{\lambda_i}^n \mathbf x.
		\end{equation}
		Taking limitation by letting $k\to \infty$ for the both sides of~\eqref{eqn:series_eigen}, we have $\mathbf {J}(\mathrm{ad}_{\xi})\mathbf x=\frac{1 - e^{-\lambda_{i}}}{\lambda_{i}}\mathbf{x}$, that is, the eigenvalues of $\mathbf {J}(\mathrm{ad}_{\xi})$ are $\frac{1 - e^{-\lambda_{i}}}{\lambda_{i}}$'s. When $ \lambda _{i}\neq 2k\pi \mathbf{i},k=\pm 1,\pm 2,\ldots$, $\frac{1 - e^{-\lambda_{i}}}{\lambda_{i}}\neq 0$, which completes the proof.
	\end{proof}

\begin{lemma}[Theorem 7 of \cite{Barrau}]\label{lemma:linear_invariant_error}
	Let $X\in G$ denote the state lying on the Lie group and $\xi:=\log\left( X\right)^\vee $ denote the vector of its Lie logarithm. If $\dot{X} = g_{u_t}(X)$ satisfies $g_{u_t}(X_1X_2)=X_1g_{u_t}(X_2)+g_{u_t}(X_1)X_2$, its lie logarithm $\xi$ satisfies a linear ordinary differential equation, i.e.,
	\begin{equation}
		\dot{\xi_t}=A_t\xi_t,
	\end{equation} 
	and $A_t$ can be calculated by linearing the mapping $g_{u_t}(exp(\xi_t))=A_t\xi_t+O(\left\|\xi \right\|^2 )$.
\end{lemma}
	}
{
	{\appendix[Matrix Lie groups useful formulas]
\begin{definition}[Jacobian of $SE_n(3)$]\label{def:Jacobians}
	For any Lie algebra $\mathbf{x}\in\mathfrak{se}_n(3)$, the left Jocabians of $\mathbf{x}$ is defined as 
	\begin{equation}
		\mathbf{J}(\ad_\mathbf{x})=\sum_{i=0}^{\infty}\frac{1}{(i+1)!}(\ad_{\mathbf{x}})^{i},
	\end{equation} 
	which is called the (left) Jacobian of $SE_n(3)$\footnote{More details about the left and right Jocobians of $
		SO(3)$ and $SE(3)$ are provided in Section 7.1.5 of~\cite{barfoot_2017}.}.
\end{definition}
Note that using the paremetrization $\mathbf{T}=\exp{(\xi^\wedge)}$ and the pertubation $\mathbf{T}'=\exp{((\xi+\delta\xi)^\wedge)}$, the logarithm of the difference~(relative to $\mathbf{T}$) can be approximated as  $\log{(\mathbf{T}'\mathbf{T}^{-1})^\vee}\approx\mathbf{J}(\ad_{\xi^\wedge})\delta \xi$ given that $\delta \xi$ is sufficiently small.
When an element of the Lie algebra $\mathfrak{se}_n(3)$ is parameterized as $\mathbf{x}^\vee:=\begin{pmatrix}
	\mathbf{\theta}^\top& v_1^\top& \cdots& v_n^\top
\end{pmatrix}^\top $ where $\theta^\wedge\in \mathfrak{so}(3)$ and $v_i\in\mathbb{R}^3$, 
the left Jacobian of $SE_n(3)$ writes
$
\mathbf{J}(\ad_{\mathbf{x}})=\begin{pmatrix}
	\mathbf{J}(\theta^\wedge)&\mathbf{0}&\cdots&\mathbf{0}\\
	\mathbf{Q}_\theta({v}_1)&\mathbf{J}(\theta^\wedge)&\cdots&\vdots\\
	\vdots&\mathbf{0}&\ddots&\mathbf{0}\\
	\mathbf{Q}_\theta(\mathbf{v}_n)&\mathbf{0}&\cdots&\mathbf{J}(\theta^\wedge)
\end{pmatrix}
$
where 
$\mathbf{J}(\theta^\wedge)=\sum_{i=0}^{\infty}\frac{1}{(i+1)!}(\theta^\wedge)^{i}=\frac{\sin\left| \theta\right|}{\left| \theta\right|}\mathbf{I}+(1-\frac{\sin\left| \theta\right|}{\left| \theta\right|})\frac{\theta\theta^\top}{\left| \theta\right|^2}+\frac{1-\cos\left| \theta\right|}{\left| \theta\right|^2}\theta^\wedge$ and $\mathbf{Q}_\theta(v_i)=\sum_{n=0}^{\infty}\sum_{m=0}^{\infty}\frac{1}{(n+m+2)!}(\theta^\wedge)^nv_i^\wedge(\theta^\wedge)^m$.
}
}	
	\bibliographystyle{unsrt}
	\bibliography{ref}

\begin{thebibliography}{10}

\bibitem{Leutenegger2015}
Stefan Leutenegger, Simon Lynen, Michael Bosse, Roland Siegwart, and Paul
  Furgale.
\newblock Keyframe-based visual–inertial odometry using nonlinear
  optimization.
\newblock {\em The Int. Journal of Robotics Research}, 34(3):314--334, 2015.

\bibitem{Campos2021}
Carlos Campos, Richard Elvira, Juan J.~Gómez Rodríguez, José M.~M. Montiel,
  and Juan~D. Tardós.
\newblock Orb-slam3: An accurate open-source library for visual,
  visual–inertial, and multimap slam.
\newblock {\em IEEE Transactions on Robotics}, 37(6):1874--1890, 2021.

\bibitem{Bloesch2017}
Michael Bloesch, Michael Burri, Sammy Omari, Marco Hutter, and Roland Siegwart.
\newblock Iterated extended kalman filter based visual-inertial odometry using
  direct photometric feedback.
\newblock {\em The Int. Journal of Robotics Research}, 36(10):1053--1072, 2017.

\bibitem{Hartley2020}
Ross Hartley, Maani Ghaffari, Ryan~M Eustice, and Jessy~W Grizzle.
\newblock Contact-aided invariant extended kalman filtering for robot state
  estimation.
\newblock {\em The Int. Journal of Robotics Research}, 39(4):402--430, 2020.

\bibitem{Geneva2020ICRA}
Patrick Geneva, Kevin Eckenhoff, Woosik Lee, Yulin Yang, and Guoquan Huang.
\newblock {OpenVINS}: A research platform for visual-inertial estimation.
\newblock In {\em Proc. of the IEEE ICRA}, Paris, France, 2020.

\bibitem{Krener2003TheCO}
Arthur~J. Krener.
\newblock The convergence of the extended kalman filter.
\newblock {\em arXiv: Optimization and Control}, pages 173--182, 2003.

\bibitem{Huang2007}
Shoudong Huang and Gamini Dissanayake.
\newblock Convergence and consistency analysis for extended kalman filter based
  slam.
\newblock {\em IEEE Transactions on Robotics}, 23(5):1036--1049, 2007.

\bibitem{Huang2008}
Guoquan~P. Huang, Anastasios~I. Mourikis, and Stergios~I. Roumeliotis.
\newblock Analysis and improvement of the consistency of extended kalman filter
  based slam.
\newblock In {\em 2008 IEEE ICRA}, pages 473--479, 2008.

\bibitem{Bonnabel2008}
Silvere Bonnabel, Philippe Martin, and Pierre Rouchon.
\newblock Symmetry-preserving observers.
\newblock {\em IEEE Transactions on Automatic Control}, 53(11):2514--2526,
  2008.

\bibitem{Bonnabel2007}
Silvere Bonnabel.
\newblock Left-invariant extended kalman filter and attitude estimation.
\newblock In {\em 2007 46th IEEE CDC}, pages 1027--1032, 2007.

\bibitem{Barrau}
Axel Barrau and Silvère Bonnabel.
\newblock The invariant extended kalman filter as a stable observer.
\newblock {\em IEEE Transactions on Automatic Control}, 62(4):1797--1812, 2017.

\bibitem{Barrau2018CDC}
Axel Barrau and Silvère Bonnabel.
\newblock Stochastic observers on lie groups: a tutorial.
\newblock In {\em 2018 IEEE CDC}, pages 1264--1269, 2018.

\bibitem{Martin2021}
Martin Brossard, Axel Barrau, Paul Chauchat, and Silvére Bonnabel.
\newblock Associating uncertainty to extended poses for on lie group imu
  preintegration with rotating earth.
\newblock {\em IEEE Transactions on Robotics}, pages 1--18, 2021.

\bibitem{Yulin2022}
Yulin Yang, Chuchu Chen, Woosik Lee, and Guoquan Huang.
\newblock Decoupled right invariant error states for consistent visual-inertial
  navigation.
\newblock {\em IEEE Robotics and Automation Letters}, 7(2):1627--1634, 2022.

\bibitem{Kanzhi2017}
Kanzhi Wu, Teng Zhang, Daobilige Su, Shoudong Huang, and Gamini Dissanayake.
\newblock An invariant-ekf vins algorithm for improving consistency.
\newblock In {\em 2017 IEEE/RSJ IROS}, pages 1578--1585, 2017.

\bibitem{barfoot_2017}
Timothy~D. Barfoot.
\newblock {\em State Estimation for Robotics}.
\newblock Cambridge University Press, 2017.

\bibitem{barrau2018invariant}
Axel Barrau and Silvere Bonnabel.
\newblock Invariant kalman filtering.
\newblock {\em Annual Review of Control, Robotics, and Autonomous Systems},
  1:237--257, 2018.

\bibitem{barrau2015non}
Axel Barrau.
\newblock {\em Non-linear state error based extended Kalman filters with
  applications to navigation}.
\newblock PhD thesis, Mines Paristech, 2015.

\bibitem{Mourikis2007}
Anastasios~I. Mourikis and Stergios~I. Roumeliotis.
\newblock A multi-state constraint kalman filter for vision-aided inertial
  navigation.
\newblock In {\em Proceedings 2007 IEEE ICRA}, pages 3565--3572, 2007.

\bibitem{Trawny2005IndirectKF}
Nikolas Trawny and Stergios~I. Roumeliotis.
\newblock Indirect kalman filter for 3 d attitude estimation.
\newblock 2005.

\bibitem{Burri25012016}
Michael Burri, Janosch Nikolic, Pascal Gohl, Thomas Schneider, Joern Rehder,
  Sammy Omari, Markus~W Achtelik, and Roland Siegwart.
\newblock The euroc micro aerial vehicle datasets.
\newblock {\em The International Journal of Robotics Research}, 2016.

\bibitem{Zichao2018}
Zichao Zhang and Davide Scaramuzza.
\newblock A tutorial on quantitative trajectory evaluation for
  visual(-inertial) odometry.
\newblock In {\em 2018 IEEE/RSJ IROS}, pages 7244--7251, 2018.

\bibitem{Hunacek2008}
Mark Hunacek.
\newblock Lie groups: an introduction through linear groups.
\newblock {\em The Mathematical Gazette}, 92:380--382, 07 2008.

\end{thebibliography}
	%
	%
	%

\end{document}